%% file: main.tex
\documentclass{article}

\usepackage[main,final]{neurips_2026}

\usepackage[utf8]{inputenc} 
\usepackage[T1]{fontenc}    
\usepackage{hyperref}       
\usepackage{url}            
\usepackage{booktabs}       
\usepackage{amsfonts}       
\usepackage{nicefrac}       
\usepackage{microtype}      
\usepackage{xcolor}         
\usepackage{amsmath}
\usepackage{graphicx}
\usepackage{enumitem}
\usepackage{cleveref}
\usepackage{amsthm}
\usepackage{subcaption}
\newtheorem{proposition}{Proposition}
\usepackage{wrapfig}

\newcommand{\mixer}{\textsc{MIX}}
\newcommand{\yrcite}[1]{\citeyearpar{#1}}

\title{Generalized Intention Modeling in Multi-Agent Reinforcement Learning}

\author{
  \textbf{Mateusz Odrowaz-Sypniewski}\quad 
  \textbf{Jasmine Bayrooti}\quad 
  \textbf{Ajay Shankar}\quad 
  \textbf{Amanda Prorok} \\
  Department of Computer Science and Technology, University of Cambridge, UK \\
  \texttt{\{msao2, jgb52, as3233, asp45\}@cst.cam.ac.uk}
}

\begin{document}

\maketitle

\begin{abstract}
Modeling an opponent's intent is critical for effective decision-making in non-cooperative, competitive, and general-sum multi-agent reinforcement learning. Existing opponent modeling methods encode intent using an embedding derived from episode information chosen a priori, such as the opponent's next action or a future environment state, and use this to guide the ego-agent's behavior. These approaches assume that the chosen information is universally representative of intent; however, we show empirically that this is not the case as intentions are often task- and environment-dependent. To address this, we introduce a task-adaptive opponent modeling framework that learns a performance-driven mixture of multiple intent representations. We further introduce a new intention representation that maximizes mutual information with the ego-agent's future returns, thereby capturing opponent information that is most directly relevant to performance. Our approach consistently matches or exceeds the performance of state-of-the-art baselines across diverse tasks and yields insights into when and why different opponent modeling strategies succeed.
\end{abstract}

\input{sections/01-intro}

\input{sections/02-related}

\input{sections/03-preliminaries}

\input{sections/04-method}

\input{sections/05-analysis}

\input{sections/06-eval}

\input{sections/07-conclusion}

\newpage


\section*{Acknowledgements}
This work is supported by European Research Council (ERC) Project 949940 (gAIa) and ARL DCIST CRA W911NF-17-2-0181. J. Bayrooti is supported by a DeepMind scholarship and the Cambridge Department of Computer Science and Technology. We gratefully acknowledge their support.

\bibliography{main}
\bibliographystyle{plainnat}

\newpage
\appendix
\input{sections/appendices}


\newpage

\end{document}

%% file: sections/01-intro.tex
\section{Introduction}
\label{sec:intro}

Optimal decision-making in multi-agent reinforcement learning (MARL) can be complicated by the absence of communication, partial observability, and the unknown intentions of other agents~\citep{albrecht_autonomous_2018}. Understanding these intentions can significantly improve the task performance of an ego-agent operating in the shared environment~\citep{pmlr-v48-he16}. For instance, the ability to forecast a predator's future trajectory can aid a prey attempting to evade it; similarly, predicting future game states in chess can inform the winning strategy. Inspired by ``theory of mind''~\citep{premack_does_1978}, the concept that humans reason about the mental states of others to predict their future behavior, a widely adopted approach is to explicitly model the other agents~\citep{albrecht_autonomous_2018}. Recent work in this space focuses on learning compact latent representations of opponents' intentions within a deep RL framework~\citep{papoudakis_agent_2021,yu2024,fu_greedy_2022,meliba}.

These representations are often trained to encode an aspect of the multi-agent episode, most commonly the opponents' next actions~\citep{pmlr-v80-grover18a, meliba, tacchetti2018relational} or a future state~\citep{yu2024}.
We refer to these aspects as \textit{episode components}: distinct elements of a MARL episode comprising agent observations, actions, and rewards. Existing methods all implicitly assume a particular episode component (or some combination) to be the most informative and learn its representation \textit{irrespective of the environment and task}. Doing so assumes that one fixed prior is invariably preferable to another, ignoring the broader question: how does the optimal representation of `intent' depend on the task itself?

Central to our work, we posit that the optimal episode component to model is the one that best informs the ego-agent's strategy and maximizes its performance, and is therefore dependent on the environment, just as the strategy itself.
Consider, as two motivating examples, the games of rock-paper-scissors and chess.
In rock-paper-scissors, predicting the opponent's next action is crucial, while modeling future environment states provides little additional benefit due to statelessness.
On the other hand, in chess, while anticipating an opponent's immediate move is certainly useful, strong play fundamentally requires reasoning about longer-term consequences and future board configurations.
Modeling intent using the same episode component for both would likely be suboptimal; intuitively, the appropriate component should depend on the environment characteristics.

To tackle this issue, we propose an adaptive architecture that explicitly models environment-dependent weightings over several definitions of intent, selecting the representation that maximizes ego-agent returns. Inspired by Mixture-of-Experts architectures~\citep{shazeer2017, jiang2024mixtralexperts}, we train separate modules that produce embeddings predictive of actions, observations, future states, and rewards, and combine them on-the-fly using a learned mixing module.
This design allows the ego-agent to redefine its notion of intent based on the environment: the weighted combinations do not require explicit interpretations and can thus represent more expressive and malleable definitions without human biases.

We also introduce a method that explicitly models the ego-agent's future rewards as a distinct episode component. We encode embeddings predictive of these rewards using a contrastive InfoNCE objective~\citep{oord2019representationlearningcontrastivepredictive} to maximize the mutual information between the embeddings and the observed future rewards. This allows us to force the representation to discard noisy state details and retain only features that are directly relevant to the ego-agent's value function.

We empirically demonstrate that our adaptive mixer architecture yields more robust and stable performance than approaches relying on any single episode component in isolation. Additionally, we establish that in some environments, modeling opponent intentions via future ego-agent rewards produces more informative representations than modeling future states or actions.
These results hold on several multi-agent benchmarks commonly used in opponent modeling, including Level-Based Foraging~\citep{christianos2020shared, papoudakis2021benchmarking}, a partially observable version of Predator-Prey~\citep{lowe2017}, Kuhn Poker~\citep{Kuhn1950}, and a customized Google Research Football~\citep{kurach2020googleresearchfootballnovel} scenario with 6 opponents, demonstrating consistent improvements over state-of-the-art baselines. In summary, our main contributions are as follows:
\begin{enumerate}[itemsep=0pt,leftmargin=*]
    \item We propose an adaptive opponent modeling architecture that combines multiple episode components via a mixing mechanism, allowing the model to prioritize the most informative components for a given environment.
    
    \item We introduce reward-predictive intention embeddings for opponent modeling, enabling agents to capture the future impact of other agents on their own returns.

    \item We demonstrate that the adaptive approach yields stronger and more stable performance across environments than modeling any single episode component in isolation (as is standard in prior opponent modeling works), and provide post-hoc analyses showing how its preferences vary by environment.

    \item We provide new data-driven insights into the performance of different opponent modeling methods, how they generalize to previously unseen opponents, and how their effectiveness depends on the opponent policies themselves.
\end{enumerate}

%% file: sections/02-related.tex
\section{Related Work}
\noindent\textbf{Opponent Modeling in Deep RL.}
Early deep MARL opponent modeling relied on restrictive auxiliary assumptions. DRON~\citep{pmlr-v48-he16} employs a Mixture-of-Experts to approximate \mbox{Q-value} functions representing different opponent strategies. However, it relies on hand-crafted features to model the opponents (limiting its scalability to complex environments) and is specific to the DQN algorithm~\mbox{\citep{mnih2013playingatarideepreinforcement}}. Self Other-Modeling (SOM)~\citep{raileanu_modeling_2018} and subsequent works~\citep{wu_too_2021, huang_efficient_2024} presume the agents always follow a specific goal from a fixed, predefined set of tasks, known as a goal library. This assumption reduces opponent modeling to simply inferring the opponent's active task; however, such information is unavailable in many general environments. In contrast, our approach generalizes to arbitrary MARL environments without relying on such assumptions and remains agnostic to the underlying learning algorithm.

\noindent\textbf{Learning Opponent Representations.}
A popular approach in opponent modeling is to model low-dimensional representations of opponent policies and use them to condition the policy of the ego-agent. Most works focus on training embeddings to be predictive of the opponents' next actions~\citep{pmlr-v80-grover18a, tacchetti2018relational}. Rabinowitz et al.~\yrcite{rabinowitz_machine_2018} propose ToM-net, an algorithm that models two embeddings per agent: one representing the long-term character and the other capturing the current state of mind. MeLIBA~\citep{meliba} implements this idea using a hierarchical VAE~\citep{Kingma2014, zhao2017learning}. OMIS~\citep{omis} uses a pre-trained transformer to predict opponent actions; instead of simply conditioning the policy on these predictions, it maintains a model of the environment and performs decision-time search, rolling out the episode in imagination.

All of the above methods share two key limitations: (i) they assume full access to the opponents' trajectories during execution, and (ii) they assume that opponent actions are universally predictive of the opponents' true intentions, regardless of the environment dynamics.
LIAM~\citep{papoudakis_agent_2021} addresses (i) via a recurrent architecture that models opponent actions and observations under partial observability and only requires access to them during training. OMG~\citep{yu2024} models embeddings predictive of future states, arguing that predicting only the immediate next action is short-sighted.
Both of these approaches address limitation (ii) only partially; they recognize that modeling embeddings based only on opponent actions may be suboptimal, but they merely shift the problem by defining their own fixed intention representations. We instead argue that the choice of what to represent is not a decision that should be made a priori, but rather an empirical question whose answer depends on the environment. This is motivated by the observation that no fixed notion of intent is most informative across environments, so our representation adapts to the task. Like LIAM, we require opponent trajectories only during training, thereby addressing both limitations.

\noindent\textbf{Single-Episode vs. Cross-Episode Modeling.}
The works discussed above assume that the ego-agent encounters a randomly sampled set of opponent policies in each episode. In contrast, cross-episode modeling assumes repeated interactions with a single, evolving opponent set. Approaches in this latter category include opponent shaping methods~\citep{foerster_learning_2018, willi_cola_2022, pmlr-v162-lu22d}, where policy updates account for their impact on other agents' learning. Other works explicitly model the opponent's evolution: GSCU~\citep{fu_greedy_2022} maintains a belief over the opponent's policy embedding conditioned on past episodes, while LILI~\citep{lili} employs a static, episode-level embedding representing the opponent's strategy in the next episode. Our work belongs to the former, more general category: we aim for robustness against a wide population of opponents rather than exploiting the learning process of any single opponent over time.

\noindent\textbf{Reward Prediction in Opponent Modeling.}
While prior work in opponent modeling has incorporated future reward prediction~\citep{lili, tacchetti2018relational, omis}, these approaches have not fully leveraged it for policy learning. For instance, Tacchetti et al.~\yrcite{tacchetti2018relational} employ return prediction only for offline analysis, whereas OMIS~\citep{omis} does not utilize its predictions to train the actor. Xie et al.~\yrcite{lili} predict rewards, but their opponent embeddings are modeled at the episode level, not the individual timestep level, and held constant throughout each episode. We propose conditioning the ego-agent on embeddings specifically predictive of the sum of future rewards across a horizon, and to train those embeddings we adopt a contrastive approach. While contrastive learning has been widely used in MARL~\citep{liu2023, lo2024learning, papoudakis_agent_2021}, we uniquely utilize this objective to predict the ego-agent's future rewards as a proxy for the opponent's impact on the ego-agent.

%% file: sections/03-preliminaries.tex
\section{Preliminaries}
\label{sec:preliminaries}
We consider $N$-agent Partially Observable Stochastic Games~\citep{shapley1953stochastic}, defined as tuples $\langle \mathcal{S}, \mathcal{A}, P, \{R^i\}_{i=1}^N, \gamma, \mathcal{O}, \{\Omega^i\}_{i=1}^N \rangle$.
Here, $\mathcal{S}$ is the set of environment states, and $\mathcal{A} = \mathcal{A}^1 \times \mathcal{A}^2 \times \dots \times \mathcal{A}^N$ is the joint action space of all agents, with $\mathcal{A}^i$ denoting the action space of agent $i$. 
The environment state transition function $P(s' \mid s, a^1, \dots, a^N)$ specifies the probability distribution over the next state given the current state and agent actions. 
$R^i : \mathcal{S} \times \mathcal{A} \to \mathbb{R}$ is the reward function for agent $i$, and $\gamma \in [0, 1)$ is the discount factor. 
Agents do not have direct access to environment states; instead, $\mathcal{O} = \mathcal{O}^1 \times \dots \times \mathcal{O}^N$ represents the joint observation space. $\Omega^i(o^i \mid s', a^1, \dots, a^N)$ describes the probability distribution over observations available to agent $i$ in the next timestep, given the next state and agent actions.

In line with previous work in opponent modeling, we assume control of a single agent, referred to as the \textit{ego-agent}, and refer to all other agents as \textit{opponents}. 
For ease of notation, we denote the ego-agent's action, observation, and reward with the superscript $1$ (e.g., $a^1, o^1, r^1$) and the joint action, observation, and reward of the opponents with the superscript $-1$ (e.g., $a^{-1}, o^{-1}, r^{-1}$). 
The policy of the ego-agent is denoted by $\pi^1(a^1 \mid o^1)$, and the joint policy of the opponents by $\pi^{-1}(a^{-1} \mid o^{-1}) = \prod_{j \neq 1} \pi^j(a^j \mid o^j)$. 
Following established practice~\citep{pmlr-v80-grover18a, tacchetti2018relational, papoudakis_agent_2021, meliba, yu2024}, we assume the existence of fixed sets $\Pi^{\text{train}}$ and $\Pi^{\text{test}}$ (not necessarily disjoint) containing stationary opponent policies. 
During training and testing, the ego-agent competes against policies sampled from $\Pi^{\text{train}}$ and $\Pi^{\text{test}}$, respectively. 
As such, the objective is to find the policy of the ego-agent $\pi^1$ that maximizes the expected cumulative discounted reward:
\begin{equation}
\label{eq:rl-objective}
    \operatorname*{argmax}_{\pi^1} \mathbb{E}_{\substack{\pi^{-1} \sim \mathcal{U}(\Pi^{\text{train}}) \\ s_{t+1} \sim P(\cdot \mid s_t, a^1_t, a^{-1}_t) \\ a^1_t \sim \pi^1(\cdot \mid o^1_t) \\ a^{-1}_t \sim \pi^{-1}(\cdot \mid o^{-1}_t)}} \left[ \sum_{t=0}^{\infty} \gamma^t R^1(s_t, a^1_t, a^{-1}_t) \right]
\end{equation}

Additionally, our theoretical analysis uses $\mathcal{H}(\cdot)$ to denote Shannon entropy. For random variables $X$ and $Y$, the mutual information is defined as $I(X;Y) = \mathcal{H}(X) - \mathcal{H}(X \mid Y)$, and similarly
$I(X;Y \mid Z) = \mathcal{H}(X \mid Z) - \mathcal{H}(X \mid Y, Z)$. To learn representations informative of future variables, we adopt the standard InfoNCE
contrastive objective~\citep{oord2019representationlearningcontrastivepredictive}. Given a latent embedding $c_t$ and a target variable $x_{t+k}$, the InfoNCE loss is defined as
\begin{equation}
\mathcal{L}_{\text{InfoNCE}} = - \mathbb{E} \left[\log \frac{\exp(f(x_{t+k}, c_t))} {\sum_{x_j \in \mathcal{X}} \exp(f(x_j, c_t))} \right],
\end{equation}
where $f(\cdot,\cdot)$ is a similarity function and $\mathcal{X}$ contains one positive sample and multiple negative samples. Minimizing this objective was proven to maximize a lower bound on the mutual information between $c_t$ and $x_{t+k}$~\citep{oord2019representationlearningcontrastivepredictive}.

%% file: sections/04-method.tex
\section{Method}

We propose \textsc{M}ixer of \textsc{I}ntention e\textsc{X}perts (\mixer{}), an adaptive opponent modeling architecture capable of adjusting which episode information regarding future opponent behavior is captured, depending on the task at hand. We begin by outlining the general \mixer{} framework in \Cref{sec:mix-arch}, and describe the design choices and training objectives of its submodules in \Cref{sec:encoders}.

\subsection{\mixer{} Architecture}
\label{sec:mix-arch}

\begin{figure}[t]
    \centering
    \includegraphics[width=\textwidth]{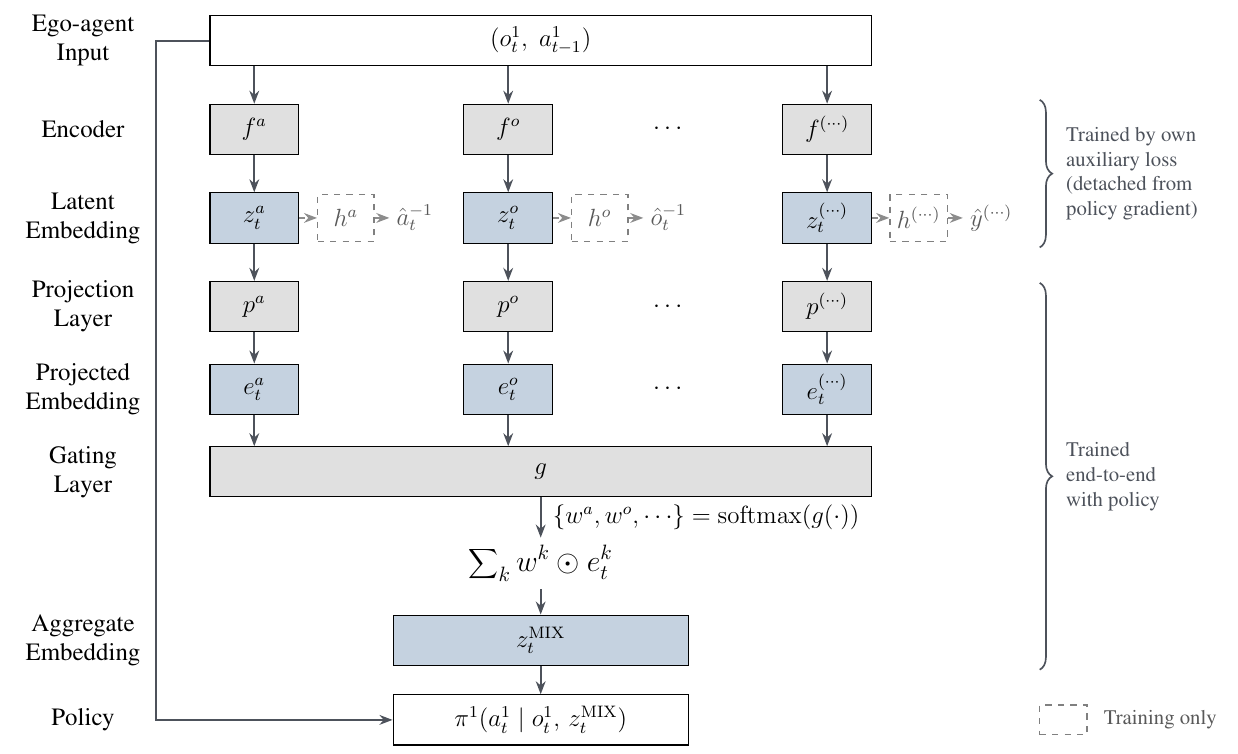}
    \caption{Overview of the \mixer{} architecture. Encoders $\{f^k\}_{k=1}^K$ produce latent representations that are combined via an MoE module into an adaptive opponent representation $z^{\mixer}$. Decoders $h^k$ and their predictions are used only during training and discarded at execution, following CTDE.}
    \label{fig:mixer_arch}
\end{figure}

In line with previous work~\citep{meliba, papoudakis_agent_2021}, \mixer{} models the latent space $\mathcal{Z}$ representative of opponent behavior. At each episode timestep $t$, it produces an embedding $z^{\mixer}_t \in \mathcal{Z}$ which then conditions the policy of the ego-agent $\pi^1(\cdot \mid o_t^1, z^{\mixer}_t)$.

To derive $z^{\mixer}_t$, we first employ a set of $K$ specialized encoders $\{f^k\}_{k=1}^K$, each taking the ego-agent's observation $o_t^1$ and previous action $a_{t-1}^1$ as input to output an embedding $z^k_t$ predictive of a distinct episode component. We detail the exact implementation and training objectives of these encoders in \Cref{sec:encoders}.
To curate the information in the embeddings and prioritize those that best capture the opponent's intentions in the current environment, we propose aggregating and weighting them using a Mixture-of-Experts (MoE)~\citep{jacobs1991} module.

While existing opponent modeling approaches typically rely on a fixed, limited subset of episode information, our use of multiple encoders allows the architecture to capture a variety of episode components informative of the opponent's future behavior. However, correctly utilizing these multiple embeddings presents a challenge. A na\"{i}ve aggregation strategy, such as simply concatenating the vectors as $[z^1_t, \dots, z^K_t]$, introduces redundant information and features that are not beneficial for policy learning. Moreover, concatenation does not scale well with the number of encoders, significantly increasing the actor input size.

The MoE module addresses these issues by dynamically weighting the input representations. It employs projection layers $\{p^k\}_{k=1}^K$ and a gating layer $g$ to map the embeddings from their respective spaces to a shared $D$-dimensional latent space. This produces ``expert'' outputs $e^k_t$, each capturing a distinct feature of the opponent's behavior.

For each encoder $k \in \{1, \dots, K\}$, the projected embedding is obtained via:
\begin{equation}
    e^k_t = \text{LayerNorm}(p^k(z^k_t)), \quad e^k_t \in \mathbb{R}^D.
\end{equation}
The gating layer $g$ processes the concatenated expert projections to generate normalized, feature-wise importance weights $w^k \in \mathbb{R}^D$ for each expert $k$.
We formulate the weighting and the final aggregation as:
\begin{align}
    [w^1, \dots, w^K] &= \text{softmax}\left( g([e^1_t, \dots, e^K_t]) \right) \\
    z^{\mixer}_t &= \sum_{k=1}^K w^k \odot e^k_t,
\end{align}
where the softmax operation is applied across the expert dimension. Crucially, the MoE module is trained end-to-end with the policy (as in prior work~\citep{papoudakis_agent_2021, yu2024}, the encoders are detached from the policy gradients).
This enables the gating mechanism to act as a learned filter, prioritizing representations useful to the policy while suppressing irrelevant auxiliary signals. The \mixer{} architecture is depicted in \Cref{fig:mixer_arch}.

\subsection{Episode Component Encoders} \label{sec:encoders}

Fundamental to the definition of $z^{\mixer}$ are the embeddings produced by the encoders, which we detail in this section. Recent work on opponent modeling highlights three episode components useful for representing intent: opponent actions~\citep{meliba}, current opponent observations~\citep{papoudakis_agent_2021}, and future environment states~\citep{yu2024}. While these components have proven effective in specific environments, we argue that they are not universally sufficient across diverse tasks. Moreover, they overlook the ego-agent's future rewards, maximizing of which is the primary MARL objective~\eqref{eq:rl-objective}. We hypothesize that conditioning the policy on embeddings predictive of future rewards allows the agent to explicitly understand how opponent behavior impacts its own performance. Consequently, we introduce future rewards as a fourth episode component to model.

We model each of these four components using a dedicated encoder-decoder architecture. The recurrent encoders ($f^a, f^o, f^s, f^r$) take the ego-agent's local observation $o^1_t$ and previous action $a^1_{t-1}$ as input and produce latent embeddings ($z^a_t, z^o_t, z^s_t, z^r_t$) that capture information about opponent actions, opponent observations, future states, and future rewards, respectively. For each encoder, we implement a corresponding decoder ($h^a, h^o, h^s, h^r$) and train each pair jointly using a dedicated objective. Each encoder-decoder pair has disjoint parameters and is updated only by its own objective, so the auxiliary losses do not interact and require no relative weighting. The ground truth targets for these decoders rely on privileged global information or future trajectory data extracted from training rollouts.
The decoders are discarded at test-time, a design that aligns with the Centralized Training Decentralized Execution (CTDE) paradigm~\citep{Oliehoek2016}.

\noindent\textbf{Action, Observation, and Future State Prediction.}
We train the encoder-decoder pairs ($f^k, h^k$) for actions, observations, and future states via MSE loss. For $k \in \{a, o, s\}$:
\begin{equation}
    \mathcal{L}_k = \mathbb{E} \left[ \| h^k(f^k(o^1_{t}, a^1_{t-1})) - y^k \|^2 \right]
\end{equation}
where the target $y^k$ is the opponent's next action $a^{-1}_t$, the opponent's current observation $o^{-1}_t$, and the future joint state $s_{t+n} \approx [o^1_{t+n}, o^{-1}_{t+n}]$ (where $n$ is a hyperparameter), respectively.

\noindent\textbf{Future Rewards Prediction.}
In addition to the episode components modeled by prior work, we also train an encoder to model embeddings predictive of future rewards. Since the relationship between local observations $o^1_t$ and future rewards is inherently more complex than predicting immediate opponent features (e.g., opponent observations $o^{-1}_t$), we use the following contrastive InfoNCE objective to maximize the mutual information between future returns over a horizon $H$ and the embedding $f^r(o^1_{t}, a^1_{t-1})$:
\begin{equation}
    \mathcal{L}_{\text{InfoNCE}} = - \mathbb{E} \left[ \log \frac{\exp \left( h^r(f^r(o^1_{t}, a^1_{t-1})) \cdot l^r(R^{H}_t) / \tau \right)}{ \sum_{j \in \mathcal{B}} \exp \left( h^r(f^r(o^1_{t}, a^1_{t-1})) \cdot l^r(R^{H}_j) / \tau \right) } \right]
\end{equation}
where $R^{H}_t = \sum_{k=1}^H r^1_{t+k}$ represents the cumulative future return, $\mathcal{B}$ denotes the training batch, $\tau$ is a temperature hyperparameter, and $l^r$ is a layer that projects the scalar return into the latent dimension.

%% file: sections/05-analysis.tex
\section{Theoretical Analysis}
\label{sec:theoretical-analysis}
In this section, we provide motivation for learning reward-predictive intention embeddings by maximizing mutual information with future returns. Specifically, we show that maximizing mutual information with future rewards minimizes an upper bound on value uncertainty.

\begin{proposition}
\label{prop:uncertainty}
Let $G_t = \sum_{k=0}^{\infty} \gamma^k r^1_{t+k+1}$ denote the infinite-horizon discounted return and $z_t^r$ be an embedding encoding intent with respect to future rewards. Also define the sequence of future rewards up to horizon $H$ as $r^1_{t+1:t+H} = (r^1_{t+1}, \dots, r^1_{t+H})$, which we take to be discrete-valued, as is the case in all environments we evaluate. Then,
\[
\mathcal{H}(G_t \mid o_t^1, z_t^r) \le \mathcal{H}(r^1_{t+1:t+H}) - I(z_t^r;r^1_{t+1:t+H}) + C,
\]
where $C$ is independent of $z_t^r$ and represents the residual tail uncertainty.
\end{proposition}

\begin{proof}
We decompose the return into a truncated component and a tail: $G_t = G_t^H + \gamma^H G_{t+H}$. By the sub-additivity of conditional entropy and the fact that conditioning reduces entropy,
\[\mathcal{H}(G_t \mid o_t^1, z_t^r) \le \mathcal{H}(G_t^H \mid o_t^1, z_t^r) + \mathcal{H}(\gamma^H G_{t+H} \mid o_t^1, z_t^r) \le \mathcal{H}(G_t^H \mid o_t^1, z_t^r) + \mathcal{H}(\gamma^H G_{t+H} \mid o_t^1).\]
Let $C$ denote the second term in RHS. For the first term, we observe that the truncated return $G_t^H$ is a deterministic function of the reward sequence $r^1_{t+1:t+H}$. Since the entropy of a function of a random variable is upper-bounded by the entropy of the variable itself, 
\[\mathcal{H}(G_t^H \mid o_t^1, z_t^r) \le \mathcal{H}(r^1_{t+1:t+H} \mid o_t^1, z_t^r).\]
Furthermore, since conditioning reduces entropy, we have $\mathcal{H}(r^1_{t+1:t+H} \mid o_t^1, z_t^r) \le \mathcal{H}(r^1_{t+1:t+H} \mid z_t^r)$. Expanding this using the definition of mutual information yields:
\[\mathcal{H}(r^1_{t+1:t+H} \mid z_t^r) = \mathcal{H}(r^1_{t+1:t+H}) - I(z_t^r; r^1_{t+1:t+H}).\]
Combining these inequalities completes the proof.
\end{proof}
As a consequence of Proposition~\ref{prop:uncertainty}, maximizing mutual information between $z_t^r$ and $r^1_{t+1:t+H}$ reduces an upper bound on value uncertainty. Since the cumulative return $R_t^H$ optimized in our practical objective (\Cref{sec:encoders}) is a deterministic function of the future reward sequence $r^1_{t+1:t+H}$, and $V^{\pi^1}(o_t^1, z_t^r)=\mathbb{E}[G_t\mid o_t^1, z_t^r]$, encouraging $z_t^r$ to be informative about $R_t^H$ directly supports stable value estimation.

%% file: sections/06-eval.tex
\section{Evaluations}
\label{sec:eval}

We compare \mixer{} against three state-of-the-art opponent modeling methods and a baseline with no explicit modeling
in the following multi-agent environments (details in Appendix~\ref{sec:environment-details}):
\begin{itemize}[leftmargin=*,topsep=0ex,itemsep=0ex,parsep=0ex]
    \item \textsc{Kuhn Poker}~\citep{Kuhn1950}: Simplified, partially observable two-player poker. Each episode lasts for $100$ steps ($33$--$50$ hands).
    \item \textsc{Partially Observable Predator-Prey (POPP)}: Modified version of Lowe et al.~\yrcite{lowe2017}, featuring partial observability for the ego-agent and a more complex reward structure~\citep{bohmer2020}. This modified version is commonly used in opponent modeling~\citep{fu_greedy_2022, papoudakis_agent_2021}.
    \item \textsc{Level-Based Foraging (LBF)}~\citep{christianos2020shared, papoudakis2021benchmarking}: Mixed-motive, fully observable environment involving two agents collecting food in a grid-world. Agents must collaborate to collect food of a higher level than their own.
    \item \textsc{Google Research Football (GRF)}~\citep{kurach2020googleresearchfootballnovel}: A modified ``Run to score with keeper'' scenario where the ego-agent faces six opponents: the keeper, two defenders between the agent and the goal, and three trailing defenders. Fully observable.
\end{itemize}

\noindent\textbf{Baselines.}
We use the following methods that each model opponent intent in a different way:
\begin{itemize}[leftmargin=*,topsep=0ex,itemsep=0ex,parsep=0ex]
    \item \textsc{LIAM}~\citep{papoudakis_agent_2021}: Uses an encoder-decoder architecture to model the opponent's next action and current observation $(a_t^{-1}, o_t^{-1})$.
    \item \textsc{MeLIBA}~\citep{meliba}: Models two embeddings per opponent using a hierarchical VAE architecture, predictive of the opponents' actions $a_{t}^{-1}$.
    \item \textsc{OMG}~\citep{yu2024}: Represents opponent intent as a future state of the environment $s_{t+k}$ modeled using a Conditional VAE.
    \item \textsc{No Opponent Modeling (NOM)}: No explicit opponent modeling, equivalent to feed-forward PPO.
    
\end{itemize}

We adapt MeLIBA and OMG to the partial observability of Kuhn Poker and POPP by restricting their inputs to local observations (details in Appendix~\ref{sec:implementation-details}).

\noindent\textbf{Evaluation Protocol.}
For all algorithms, we use PPO~\citep{schulman2017proximalpolicyoptimizationalgorithms} as the backbone for policy learning. Except for GRF, we evaluate performance in two settings: \textit{seen}, where algorithms are evaluated against the set of opponents encountered during training ($\Pi^{\text{train}}=\Pi^{\text{test}}$); and \textit{unseen}, testing generalization to novel opponents ($\Pi^{\text{train}} \cap \Pi^{\text{test}} = \emptyset$). In GRF, we evaluate exclusively in the \textit{seen} setting, since the opponents are controlled by the built-in AI at default difficulty, which has no natural train/test partition. Reported results are averaged over multiple random seeds, with shaded areas denoting one standard deviation. We give full implementation details in Appendix~\ref{sec:implementation-details}, and evaluate the computational complexity and runtime of \mixer{} in Appendix~\ref{sec:complexity-analysis}.

In Kuhn Poker, opponents follow coded heuristics. For POPP and LBF, we pretrain opponents using MAPPO~\citep{yu2022} to produce the train and test sets. Similar to Jing et al.~\yrcite{omis}, we introduce an auxiliary diversity objective that discourages homogeneity in policies within these sets, and thus produces more challenging opponents. We evaluate three opponent settings: no explicit diversity, low, and high diversity (details in Appendix~\ref{sec:opponent-policies}).

\subsection{Performance Against Baselines}
\label{sec:perf-against-baselines}

\textbf{Kuhn Poker.}
\Cref{fig:kp_results} shows evaluations against seen (left) and unseen (right) opponents, averaged over $10$ random seeds.
In both settings, \mixer{} and LIAM outperform the rest of the baselines. Given that the environment features very short-term dynamics (with a single hand lasting $2$--$3$ timesteps), LIAM's strong performance is expected, as it prioritizes immediate episode components such as the opponent's next action and current observation. \mixer{} is able to effectively adapt to focus on these short-term qualities.

\begin{wrapfigure}{r}{0.5\textwidth}
    \vspace{-15pt}
    \centering
    \includegraphics[width=\linewidth]{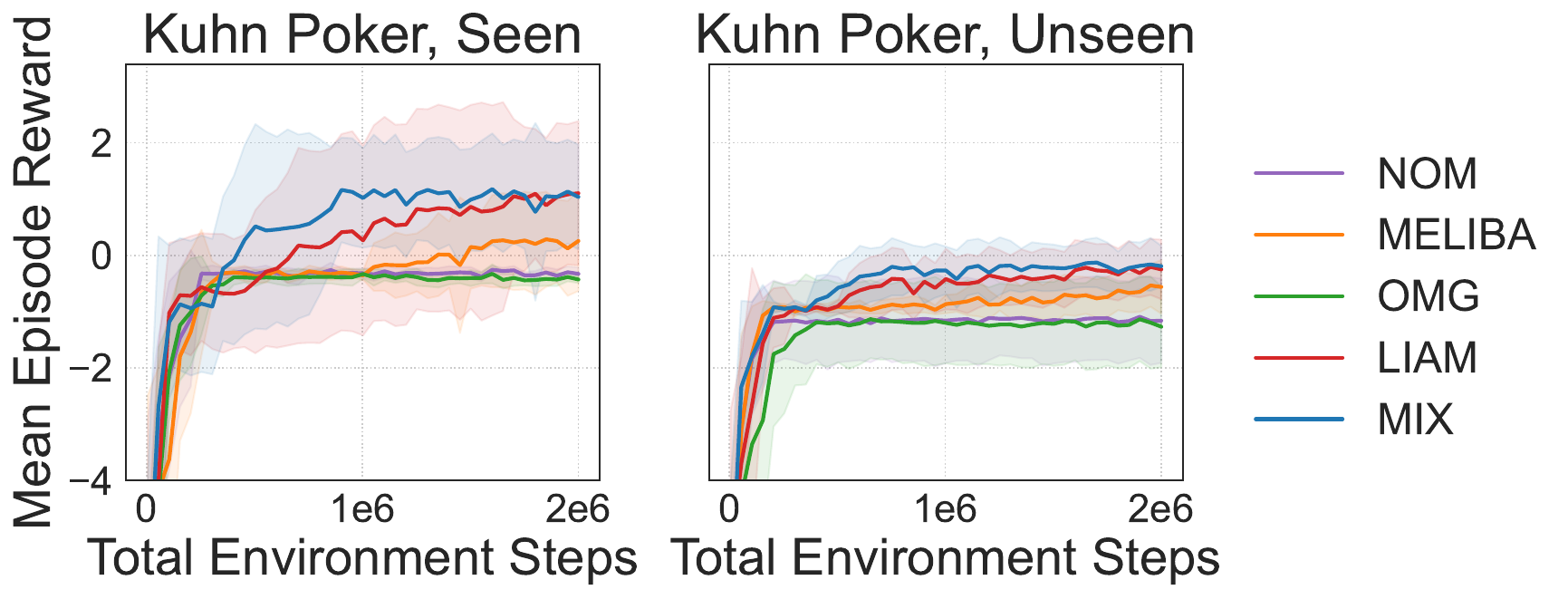}
    \caption{\small Kuhn Poker performance against \textit{seen} (left) and \textit{unseen} (right) opponents, averaged over 10 seeds.}
    \label{fig:kp_results}
    \vspace{-10pt}
\end{wrapfigure}

\noindent\textbf{POPP and LBF.}
\Cref{fig:performance-comparison} (left) shows algorithm performance against opponents seen during training. We see that \mixer{} broadly outperforms all baselines, most significantly in POPP, and is matched only by OMG in the no- and low-diversity settings of LBF. Conversely, OMG achieves relatively poor performance in POPP, failing to match even the NOM baseline in the no-diversity setting; we attribute this to its reliance on future global state reconstruction, which degrades under partial observability. While LIAM performed well in Kuhn Poker, it seems to struggle in environments with a more complex reward structure. We note that increased opponent diversity generally makes the performance gap between algorithms more pronounced, highlighting the importance of accurate opponent modeling in settings with heterogeneous agent strategies.

\begin{figure*}[tb]
    \centering
    \includegraphics[width=\textwidth]{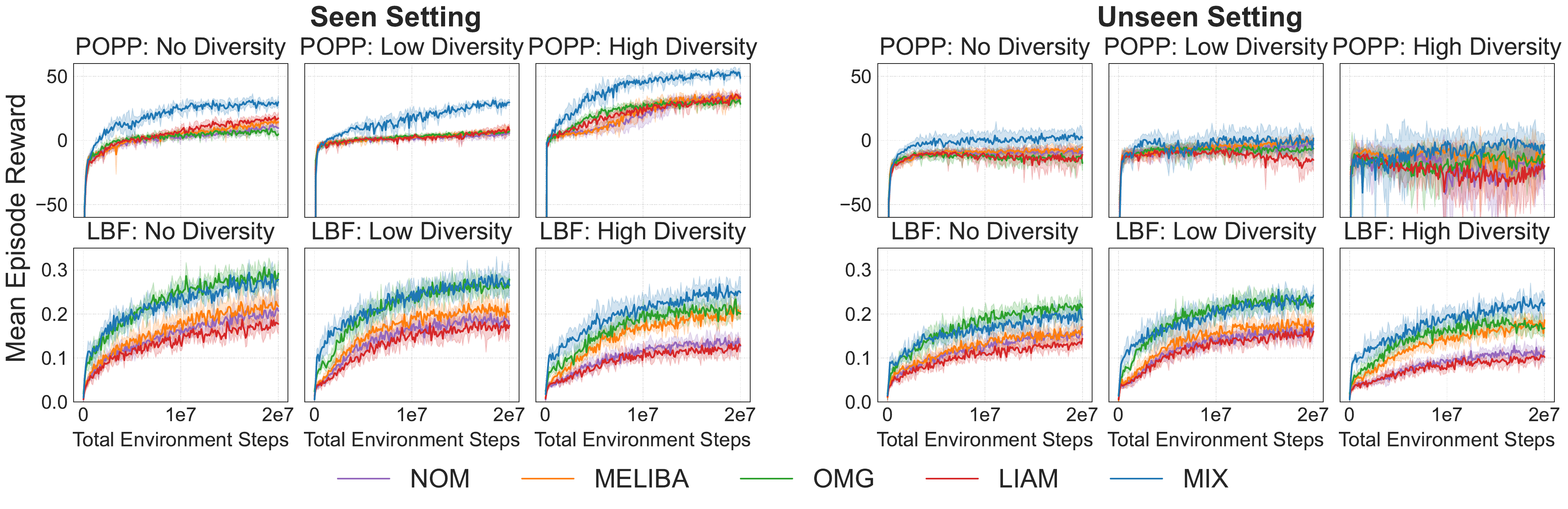}
    \caption{
    Comparisons against baseline methods in six environment configurations for \textit{seen} (left) and \textit{unseen} (right) opponent settings. \mixer{} broadly matches or outperforms all baselines, and the improvement becomes most prominent for highly diverse opponent sets. Results averaged over $10$ seeds for unseen POPP, and $5$ for all others.
    }
    \label{fig:performance-comparison}
\end{figure*}

\begin{wrapfigure}{r}{0.5\textwidth}
    \vspace{-12pt}
    \centering
    \includegraphics[width=\linewidth]{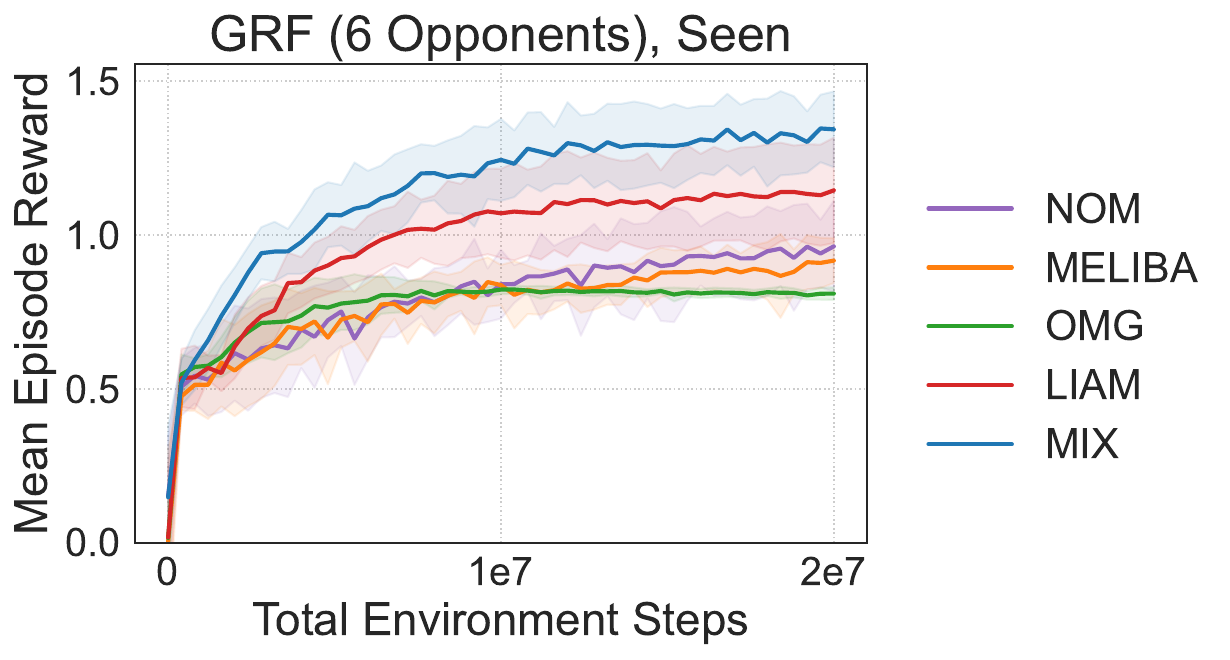}
    \caption{\small GRF performance against 6 built-in AI opponents. Results averaged over 20 random seeds.}
    \label{fig:grf}
    \vspace{-12pt}
\end{wrapfigure}

The results with unseen opponents, shown in \Cref{fig:performance-comparison} (right), exhibit a similar pattern, albeit with higher noise. As expected, we observe a general decrease in performance across all algorithms due to the distribution shift between training and evaluation policies. \mixer{} generally matches or exceeds the baselines, only slightly outperformed by OMG in no-diversity LBF.

\noindent\textbf{Scaling to GRF.}
In \Cref{fig:grf}, we evaluate \mixer{} on GRF, an environment with more complex dynamics and six opponents, as well as substantially larger observation and action spaces (115 and 19 dimensions, respectively). We see that \mixer{} noticeably outperforms all baselines, improving episode rewards and sample efficiency. This evidences effective scaling performance and goes beyond simple environments typically tested in prior research.

Curiously, we also observe that in spite of full observability of this environment, OMG is not as competitive as it was in LBF (also fully observable). This suggests that the algorithm's approach does not scale as effectively to environments with more complex, higher-dimensional dynamics.

\subsection{Evaluation of \mixer{} Embeddings}
\label{sec:embedding-eval}

\noindent\textbf{Recovering Opponent Identity.} In \Cref{sec:encoders} we hypothesize that embeddings predictive of the ego-agent's own future rewards are informative for opponent modeling. This is not obvious: future rewards clearly carry information about the ego-agent's performance, but it is not immediately evident whether they can distinguish between the opponents' strategies. To test this, we train a small linear classifier on frozen embeddings, both the future-reward encoder's $z^r$ and the aggregate $z^{\mixer{}}$, to predict which opponent the ego-agent is facing. As a control, we apply the same probe to a randomly initialized encoder of architecture identical to $f^r$, since even an untrained recurrent encoder compresses the trajectory in ways a classifier can exploit. \Cref{fig:opponent_identity} reports accuracy across timesteps for three environments; GRF is excluded, as all its opponents follow the same built-in heuristic. Both learned embeddings identify the opponent well above chance and consistently above the random-encoder control across all three environments, supporting our hypothesis; in Kuhn Poker the control remains near chance, indicating that essentially all of the recoverable opponent information there comes from the embeddings' training rather than from recurrent compression alone.

\begin{figure*}[tb]
    \centering
    \includegraphics[width=1.0\textwidth]{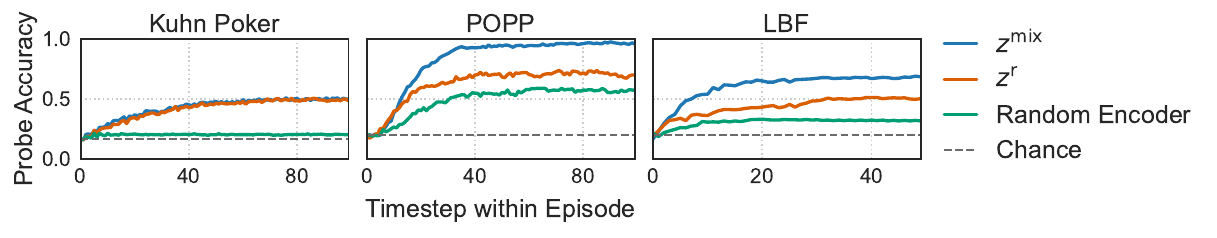}
    \caption{Opponent-identity classifier accuracy over the course of an episode, reported on 400 held-out episodes, with the classifier trained on a disjoint set of 1600 episodes. A randomly initialized encoder of the same architecture serves as a control. Configurations: Kuhn Poker (seen), POPP (seen, no diversity), LBF (seen, high diversity).}
    \label{fig:opponent_identity}
\end{figure*}

\noindent\textbf{Encoder Importance.} To analyze the relative importance of each encoder in \mixer{} to the final embedding, we apply the first-order Gradient$\times$Input attribution method~\citep{shrikumar2017justblackboxlearning}. \Cref{fig:intention_importance} shows that all experts actively contribute, generally accounting for between $10\%$ and $45\%$ of the total magnitude, and no single expert leads across all environments. While attribution patterns are similar for some settings (e.g., LBF and GRF), they differ drastically for others (e.g., LBF vs. Kuhn Poker). In two of the four environments, our Future Rewards expert consistently ranks as the most important. In Kuhn Poker, since the objective is to infer the opponent's hidden card, importance shifts toward the Opponent Observations and Future States experts. 
This variability in contributions highlights how \mixer{} adapts to the needs of different tasks. Appendix~\ref{sec:individual-expert-analysis} further shows that policies conditioned using MIX match or outperform those conditioned on any individual expert in three of four environments, and outperform all but one in Kuhn Poker. In contrast, every individual expert performs poorly in at least one environment. Collectively, these analyses validate our central hypothesis that relying on any single episode component across different environments is suboptimal, necessitating an adaptive approach.

\begin{figure*}[tb]
    \centering
    \includegraphics[width=1.0\textwidth]{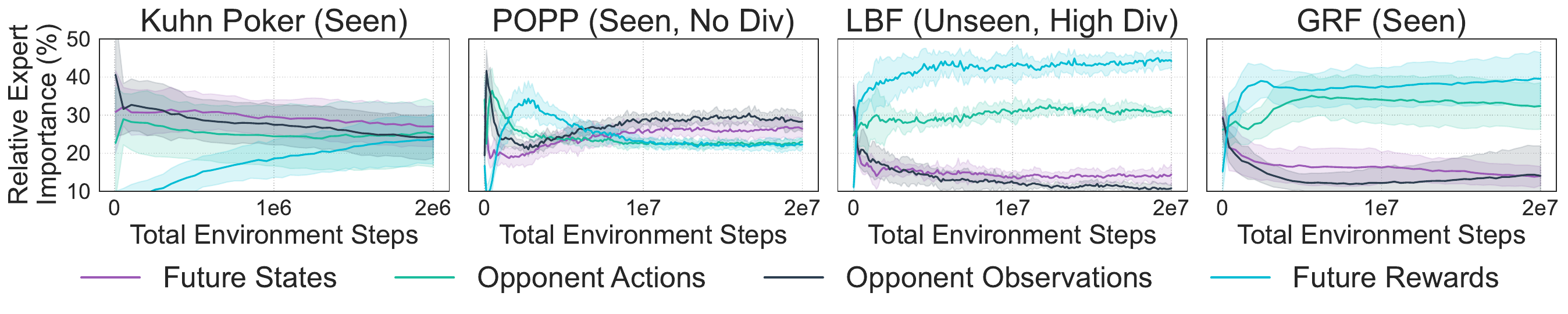}
    \caption{
    Relative expert importance, calculated as the magnitude of Gradient$\times$Input attribution expressed as a percentage of the total. Expert influence dynamically adapts over the course of training and varies across environments. Results averaged over multiple random seeds (KP: $10$, POPP/LBF: $5$, GRF: $20$).
    }
    \label{fig:intention_importance}
    \vspace{-12pt}
\end{figure*}

\subsection{Ablations}
\label{sec:ablations}

We present additional ablation studies to establish two key architectural choices we employ in \mixer{}.
First, in \Cref{sec:encoders}, we proposed the InfoNCE objective over a simpler MSE loss to predict future rewards, hypothesizing that a contrastive objective is superior in capturing the intent encoded within rewards. We also chose to model embeddings predictive of the scalar sum $R^{H}_t = \sum_{k=1}^H r^1_{t+k}$ rather than the full sequence $r^1_{t+1:t+H}$. Since $R^{H}_t$ is a deterministic function of the sequence, the data-processing inequality gives $I(z^r_t; R^{H}_t) \leq I(z^r_t; r^1_{t+1:t+H})$, so the scalar objective is a lower bound on the sequence objective; in exchange, it keeps the projection layer's input dimension fixed rather than scaling with the horizon. 
To validate both choices, we train two baseline future rewards encoder-decoders: one using an MSE loss on the same scalar target, and one using InfoNCE with the full reward sequence as the target (all other settings unchanged).
\Cref{fig:ablations} (left) compares the backbone PPO policy conditioned on embeddings from each, in a representative POPP environment. The MSE objective performs noticeably worse, affirming our choice of a more informative objective. The scalar and sequence InfoNCE variants perform similarly, with the sequence variant achieving only slightly higher mean rewards. This suggests the information discarded by summation is not essential for policy learning; we therefore adopt the scalar target for architectural simplicity.

\begin{wrapfigure}{r}{0.5\textwidth}
    \vspace{-6pt}
    \centering
    \includegraphics[width=\linewidth]{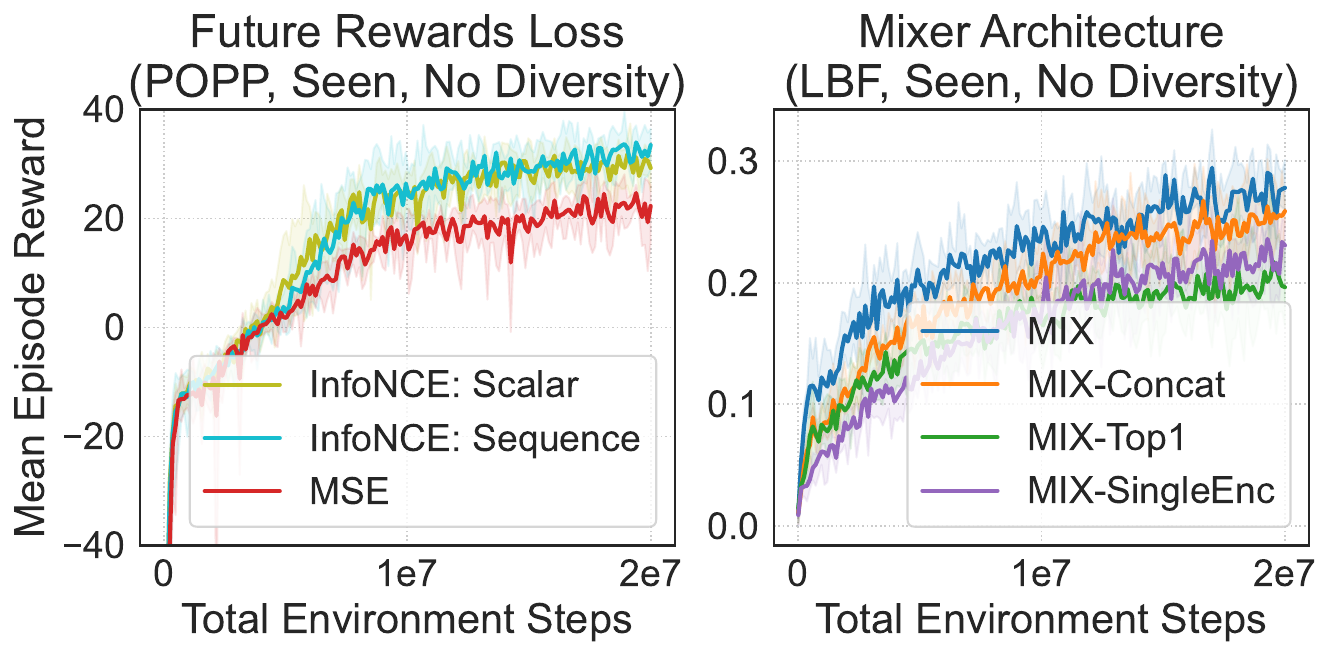}
    \caption{
      Ablations in the seen setting with no explicit opponent diversity. (Left) Future Rewards objective on POPP. (Right) MIX architecture on LBF. Results averaged over 5 random seeds.
    }
    \label{fig:ablations}
    \vspace{-12pt}
\end{wrapfigure}

Second, we wish to establish that \mixer{} outperforms other baselines not merely due to auxiliary supervision or the availability of multiple encoder embeddings, but rather through its dynamic information aggregation mechanism.
To do so, we train three alternatives:
\begin{itemize}[itemsep=2pt,topsep=4pt,leftmargin=*]
    \item \textsc{MIX-SingleEnc}, a variant of MIX with only a single encoder (parameter-matched to MIX's four), trained with all four prediction objectives simultaneously;
    \item \textsc{MIX-Concat}, which removes the gating mechanism entirely and instead na\"{i}vely concatenates all available embeddings; and
    \item \textsc{MIX-Top1}, which replaces our feature-wise gating with a Top-K mechanism~\citep{shazeer2017} with $K=1$, masking all but the largest gating logit so that the network must select a single embedding at each timestep rather than learn a soft mixture.
\end{itemize}
\Cref{fig:ablations} (right) shows that \mixer{}'s soft MoE outperforms all three alternatives on a sample LBF environment. We present additional ablations in Appendix~\ref{sec:additional-ablations}.

%% file: sections/07-conclusion.tex
\section{Conclusion and Limitations}
\label{sec:conclusion}

\noindent\textbf{Conclusion.} In this work, we empirically demonstrate that the performance of baselines relying on a fixed representation of opponent intent depends heavily on the environment, highlighting that such an approach is suboptimal as the effectiveness of specific intention representations is highly task-dependent. We show that performance can instead be substantially improved by dynamically adapting the modeled episode information to the environment at hand. To this end, we introduce \mixer{}, a framework capable of this dynamic adaptation based on the current setting. As a key component, we also introduce an encoder trained to model the ego-agent's future rewards via mutual information maximization. We validate our architecture across several diverse multi-agent environments, demonstrating superior results and generalization to novel opponents not encountered during training.

\noindent\textbf{Limitations and Future Work.} One limitation of our work is our assumption that the ego-agent faces a randomly sampled set of opponent policies per interaction. While this was a conscious decision to prioritize robustness against a wide population of opponents, in future work we aim to apply our adaptive representations to cross-episode settings where opponents actively evolve between interactions. Furthermore, our framework assumes a single ego-agent maintaining an opponent model over all other agents jointly. Extending this to team-based settings, where multiple cooperative agents each maintain and potentially share adaptive representations of the opposing team, introduces new challenges around representation alignment that will be explored in future work.

%% file: sections/appendices.tex
\section{Environment Implementation Details} \label{sec:environment-details}
All environments use discrete action spaces.
For Kuhn Poker, we train for $2$ million steps, evaluating every $50$ thousand steps over $10{,}000$ episodes.
For POPP, LBF, and GRF, we train for $20$ million environment steps, evaluating over $200$ episodes every $120$ thousand steps (POPP and LBF) or $400$ thousand steps (GRF).

\subsection{Kuhn Poker}

We adapt the open-source Kuhn Poker implementation of OpenSpiel\footnote{\url{https://github.com/google-deepmind/open_spiel}, Apache-2.0 License} \citep{LanctotEtAl2019OpenSpiel}.

Kuhn Poker features two players, $P1$ and $P2$. Each hand begins with the players receiving a private card from a set of three possible cards. The players take turns sequentially to either check/fold or bet/call in a single betting round, which results in five possible action sequences per hand:
\begin{itemize}
    \item Check, Check
    \item Check, Bet, Fold
    \item Check, Bet, Call
    \item Bet, Fold
    \item Bet, Call
\end{itemize}

At the end of a hand, the winner receives a reward of +2 if both players have chosen to bet (or call), and +1 otherwise. If a player folds, the non-folding player wins; otherwise, the player with the higher card wins. The other player receives the corresponding negative reward. In our setting, the ego-agent is always player $P1$, and the opponent $P2$.

The observation space in the original implementation consists of the player index (one-hot), the private card (one-hot), and the hand history (with each turn one-hot encoded).

To allow for learning opponent strategies over time, we modify the original implementation to play multiple hands in one episode, with a single episode lasting 100 steps (between 33 and 50 hands). We also modify the observation space such that when a new hand begins, the agents can observe their opponent's card from the previous hand.

\subsection{Partially Observable Predator-Prey}
We adopt the specific environment configuration provided by the open-source LIAM implementation\footnote{\url{https://github.com/uoe-agents/LIAM}, MIT License} \citep{papoudakis_agent_2021}.

The environment comprises one prey agent, three predator agents, and two randomly placed large obstacles. The prey operates under partial observability; it can only observe the relative positions of predators and obstacles if they are within a specific range ($0.5$). The predators maintain full observability of the prey's relative position and velocity, as well as the positions of other predators.

The prey's reward structure is defined as follows:
\begin{itemize}
    \item Positive reward (+1): If the prey collides with \textit{exactly one} predator in a timestep.
    \item Negative reward (-1): If the prey collides with \textit{multiple} predators in a timestep.
\end{itemize}
The prey is also penalized for going out of bounds. Similarly, the predators receive a positive reward only if \textit{multiple} predators collide with the prey simultaneously; otherwise, they receive a negative reward. Episodes last for 100 steps. A rendering of the environment is depicted in \Cref{fig:popp}.

\begin{figure}[htbp]
  \centering
  \begin{subfigure}[t]{0.4\textwidth}
    \centering
    \includegraphics[width=\linewidth]{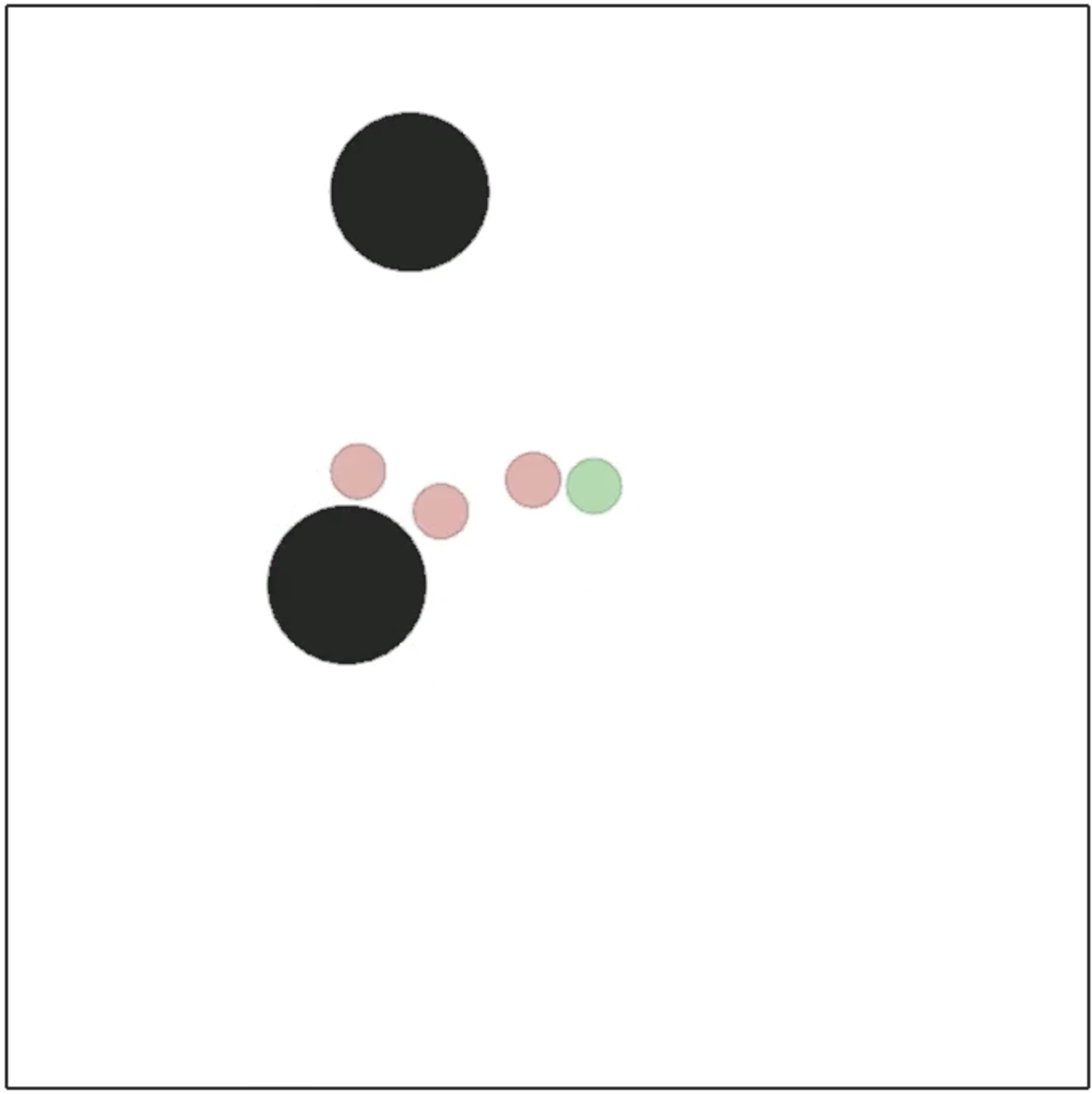}
    \caption{
      Rendering of the Partially Observable Predator-Prey environment. The prey agent (green) aims to collide with exactly one predator (red), but must learn to avoid being captured by multiple predators at once. The black circles represent obstacles.
    }
    \label{fig:popp}
  \end{subfigure}
  \hfill
  \begin{subfigure}[t]{0.4\textwidth}
    \centering
    \includegraphics[width=\linewidth]{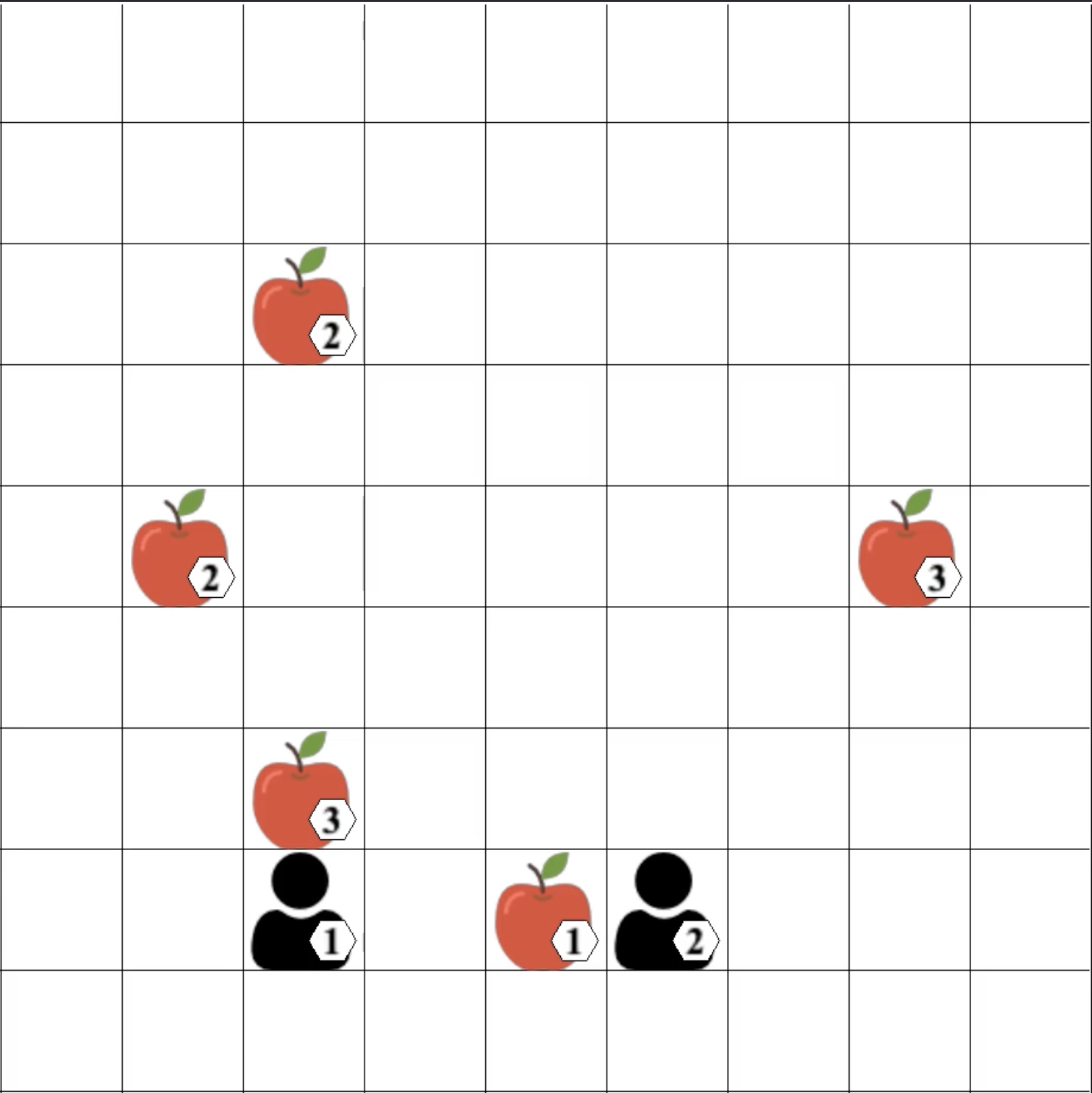}
    \caption{
      Rendering of the Level-Based Foraging environment. Two agents compete to collect food items (apples) but must collaborate to collect food with a level higher than their own. 
    }
    \label{fig:lbf}
  \end{subfigure}
  
  \caption{Renderings of the POPP and LBF environments.}
  \label{fig:both_envs}
\end{figure}

\subsection{Level-Based Foraging}

We use the open-source Level-Based Foraging environment implementation\footnote{\url{https://github.com/semitable/lb-foraging}, MIT License} \citep{christianos2020shared, papoudakis2021benchmarking}.

The environment features two players in a $9\times9$ grid, competing to collect 5 randomly placed food items, with full observability. Both agents and food items are assigned a random level; agents can collect food of a lower level individually but must collaborate to collect food of a higher level than their own. Episodes last for a maximum of 50 timesteps, terminating early if all food is collected. The reward for successfully collecting a food item is calculated as $\frac{\text{food level}}{\text{sum of all food levels}}$, split among the agents participating in the collection in proportion to their levels.

A rendering of the environment is presented in \Cref{fig:lbf}.

\subsection{Google Research Football}

We use the open-source GRF implementation\footnote{\url{https://github.com/google-research/football}, Apache-2.0 License} \citep{kurach2020googleresearchfootballnovel}.

We evaluate on a slightly modified version of the ``Run to Score with Keeper'' scenario. In the original, the ego-agent starts with the ball and aims to score against a keeper while being chased by five defenders from behind. To increase complexity for the ego-agent, we move two of those defenders to start between the ego-agent and the goal; this is the only structural modification. We utilize both checkpoint and scoring rewards: the ego-agent receives a $\pm 1$ reward for scoring or being scored against, and an additional $0.1$ reward each time the ball is first possessed in one of $10$ designated sections on the opponent's side of the field. This yields a maximum total reward of $2.0$ per episode. The episode lasts for $200$ steps, terminating early if a goal is scored or if the ego-agent loses possession of the ball. A rendering of the original ``Run to Score with Keeper'' scenario and our modified version is depicted in \Cref{fig:grf_render}.

\begin{figure}[htbp]
  \centering
  \begin{subfigure}[t]{0.49\textwidth}
    \centering
    \includegraphics[width=\linewidth]{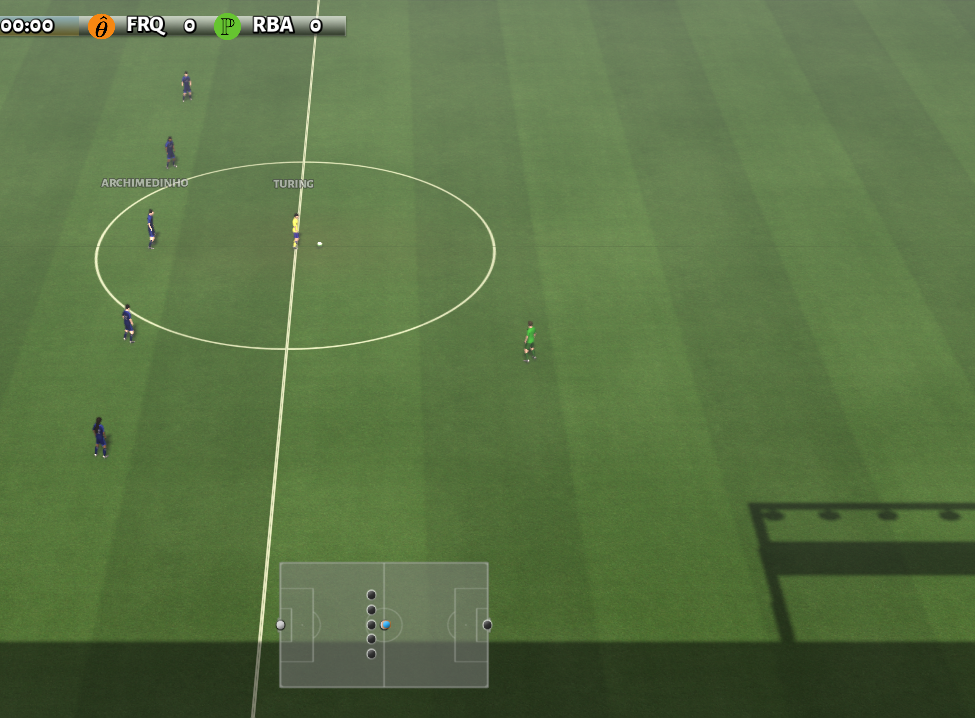}
  \end{subfigure}
  \hfill
  \begin{subfigure}[t]{0.49\textwidth}
    \centering
    \includegraphics[width=\linewidth]{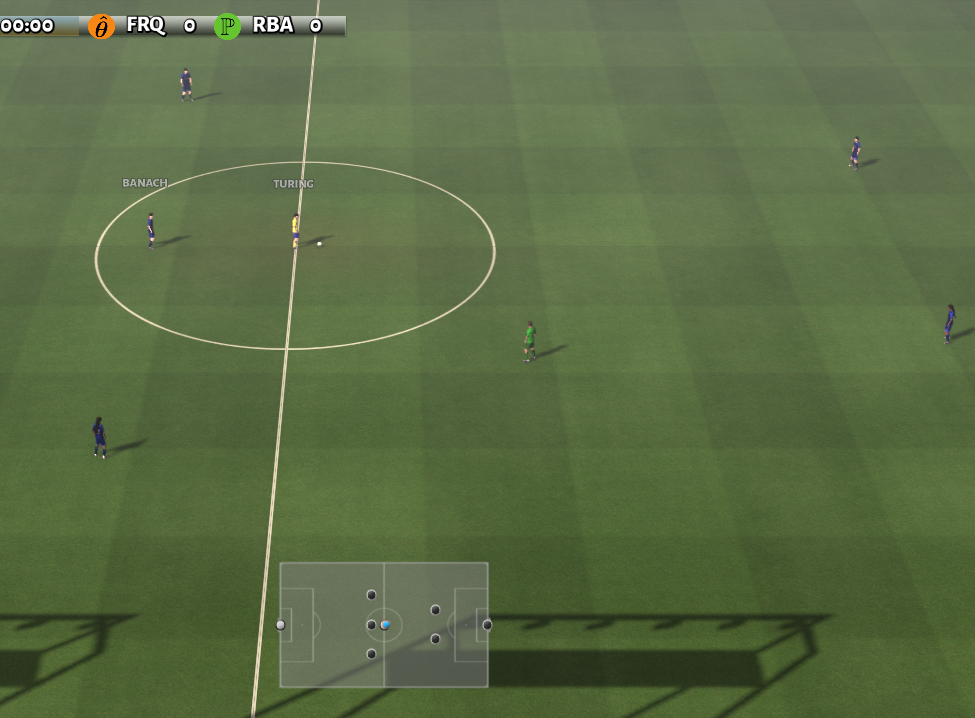}
  \end{subfigure}
  
  \caption{(Left) Rendering of the unmodified ``Run to Score with Keeper'' scenario. (Right) Rendering of the modified scenario used for evaluation. For added complexity, two of the defenders now start between the ego-agent and the goal. In both scenarios, the ego-agent begins with possession of the ball, with the objective of scoring against the keeper while avoiding defenders.}
  \label{fig:grf_render}
\end{figure}

\section{Opponent Policies}
\label{sec:opponent-policies}

\begin{figure}
  \centering
    \includegraphics[width=0.4\textwidth]
        {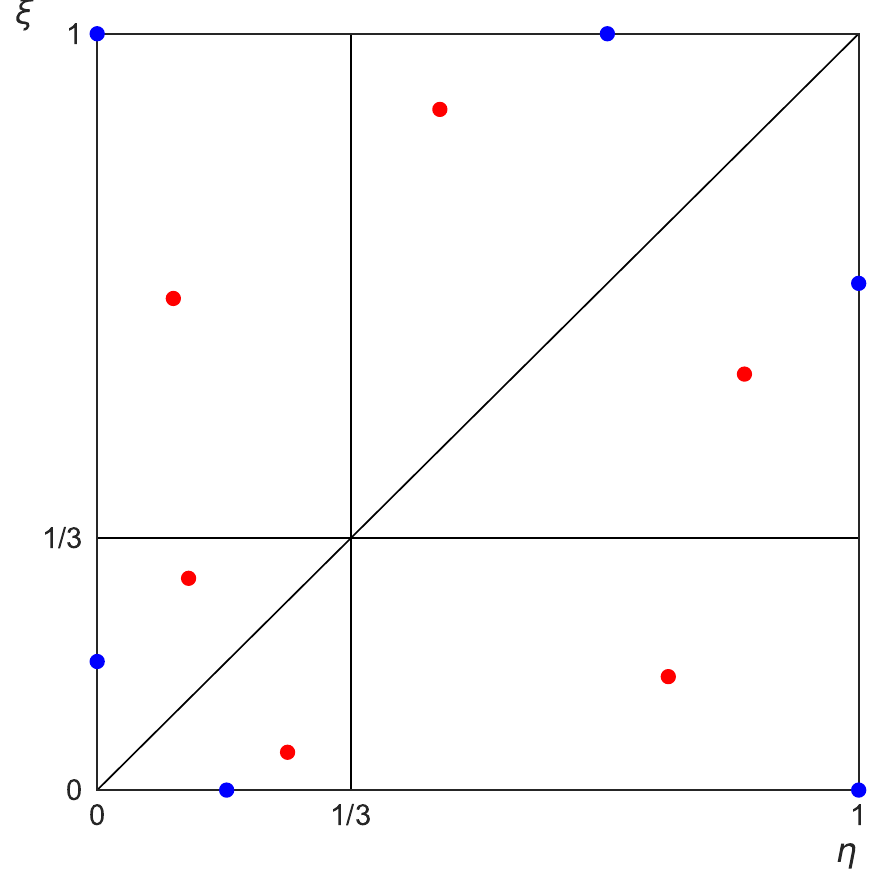}
    \caption{
    Visualization of the $P2$ strategy space in Kuhn Poker, with the $x$, $y$ axes corresponding to $\eta$, $\xi$ parameter values. We mark the $\Pi^{\text{train}}$ heuristics as blue points, and the randomly sampled heuristics for the unseen setting as red points.
    }
    \label{fig:kuhn_strategies}
\end{figure}

In our implementation of Kuhn Poker, the opponent always plays as the second player ($P2$). Viable $P2$ strategies can be parameterized by two variables, $\eta$ and $\xi$, which govern the probability of betting at two specific decision points (facing a pass with a Jack, and facing a bet with a Queen)~\citep{southey2009}. The optimal strategy lies at $\eta = 1/3$ and $\xi = 1/3$, and for all other values, there exists a single winning strategy for player 1 ($P1$). \Cref{fig:kuhn_strategies} depicts the partitioning of the $P2$ strategy space into regions where strategies share the same winning $P1$ counter-strategy. In blue, we mark the six $(\eta, \xi)$ configurations we selected for the training heuristic set, $\Pi^{\text{train}}$. For the unseen opponent setting, we randomly sampled six more heuristics, one for each $P2$ strategy region; we mark those in red.

Opponent strategies in the POPP and LBF environments are pre-trained using the MAPPO algorithm~\citep{yu2022} with no parameter sharing. 
Prior work has established that the diversity of the opponent policies encountered has a high impact on opponent modeling performance~\citep{fu_greedy_2022, omis, jing2025}. Therefore, as detailed in \Cref{sec:eval}, we generate multiple sets of training and testing opponents per environment, each with a different level of diversity. 
Inspired by OMIS~\citep{omis}, we adapt the Maximum Entropy Population-Based Training algorithm~\citep{zhao2023}. We modify the original implementation\footnote{\url{https://github.com/ruizhaogit/maximum_entropy_population_based_training}, MIT License} to train opponents sequentially rather than in a round-robin fashion. This process maximizes an auxiliary objective: the divergence between the action distribution of the agent currently being trained and the population mean of the previously trained agents. We control the strength of this objective via a hyperparameter multiplier $\alpha$. Specifically, for the low and high diversity settings, we use $\alpha=0.0005$ and $0.001$ in LBF, and $\alpha=0.005$ and $0.01$ in POPP, respectively.

Finally, opponents in GRF are controlled by the standard built-in AI with the default, medium difficulty ($\theta=0.6$).

\section{Additional Ablations}
\label{sec:additional-ablations}
\subsection{Comparison of Individual Expert Embeddings}
\label{sec:individual-expert-analysis}

\begin{figure}
  \centering
    \includegraphics[width=1.0\textwidth]
        {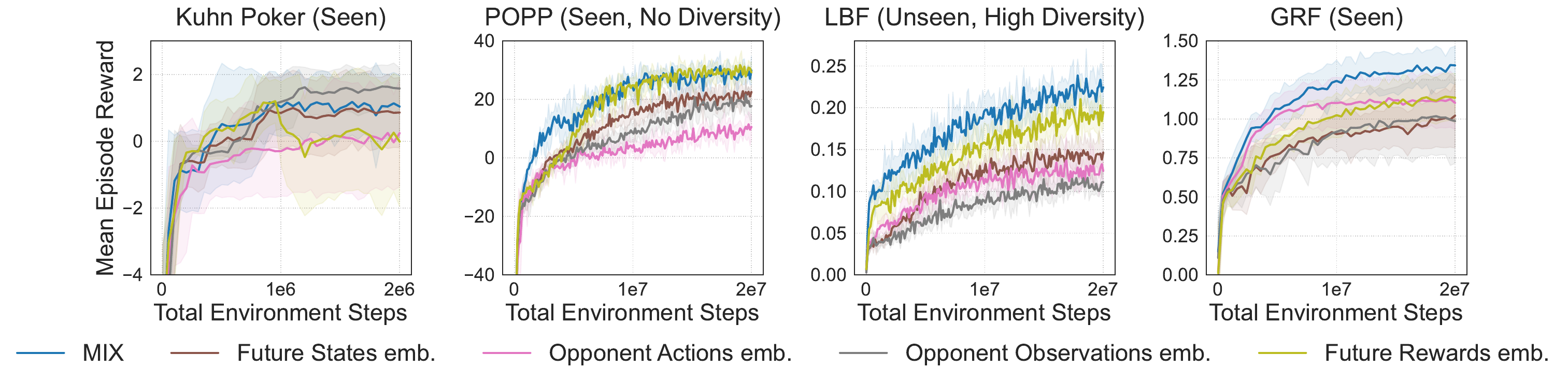}
    \caption{
    Comparing \mixer{} against baselines that model individual expert embeddings in four environments. Results averaged over multiple random seeds (KP: $10$,
    POPP/LBF: $5$, GRF: $20$).
    }
    \label{fig:intention_ablation}
\end{figure}

Ideally, \mixer{} should perform at least as well as a policy conditioned on any of its components. To validate this, we compare \mixer{} against four new baselines where the backbone PPO policy is directly conditioned on a single raw embedding $z^k \in \{z^a, z^o, z^s, z^r\}$ without using the MoE module. 

\Cref{fig:intention_ablation} presents this evaluation for representative configurations across the four environments. We observe that the effectiveness of individual baselines is highly environment-dependent. For example, Future Rewards dominates in POPP and LBF but struggles in Kuhn Poker, where Opponent Observations and Future States excel. Meanwhile, Opponent Actions generally performs poorly but proves highly competitive in GRF. In general, each baseline performs poorly in at least one environment. At the same time, \mixer{} matches the best individual expert in POPP and outperforms it in LBF and GRF, only falling slightly behind the Opponent Observations expert in Kuhn Poker. This further demonstrates that \mixer{} effectively synthesizes information from multiple experts, performing consistently well across environments in a way no single expert does.

\subsection{Value of Auxiliary Supervision}

We wish to address two questions that may arise regarding \mixer{}'s design. First, each of \mixer{}'s four encoders is trained with an auxiliary prediction objective, and auxiliary tasks can improve learning on their own; it is worth asking how much of \mixer{}'s advantage comes from this supervision rather than the architecture itself. Second, in fully observable environments the ego-agent is able to access the opponent's state in its observation, raising the question of whether \mixer{}'s embedding collapses to a redundant re-encoding of information the ego-agent already has. We discuss both points below.

\begin{figure}
  \centering
    \includegraphics[width=0.60\textwidth]
        {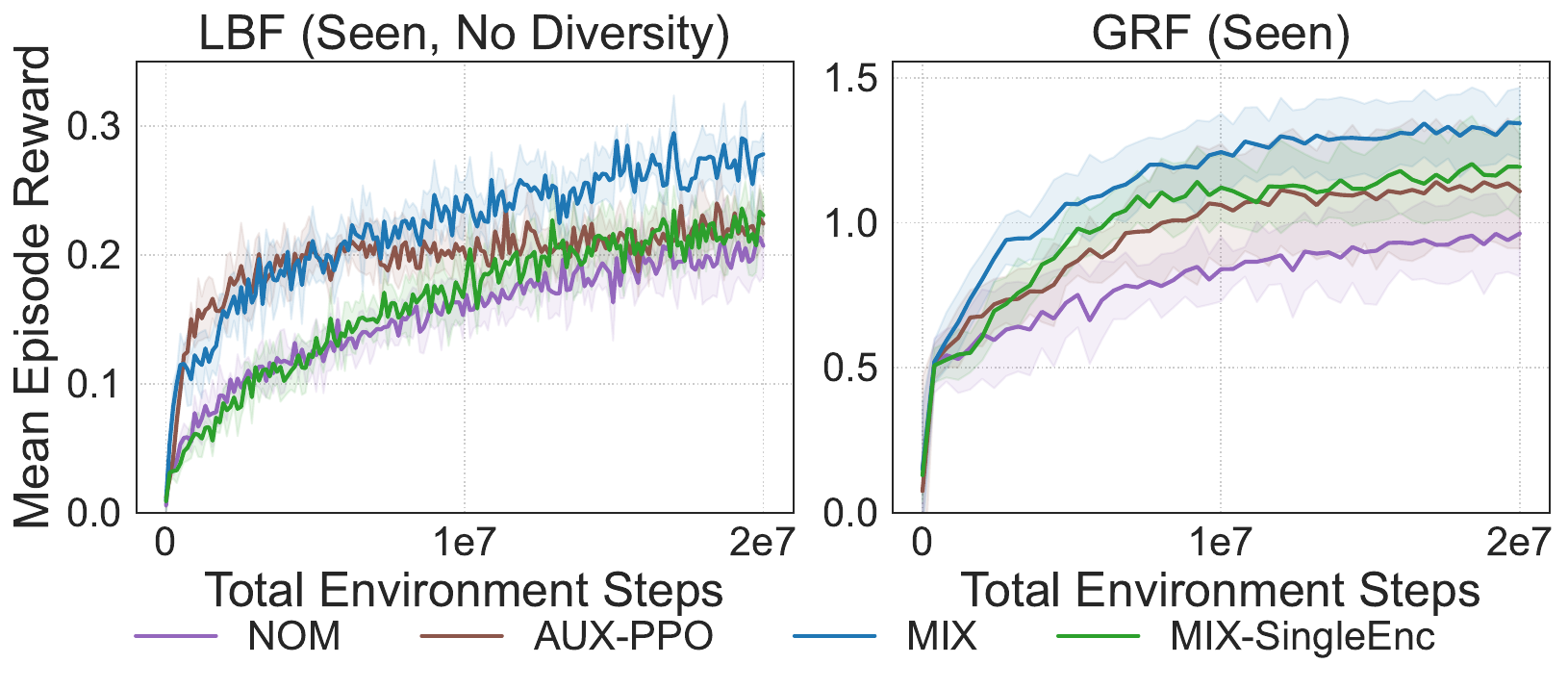}
    \caption{
    Comparing \mixer{} against baselines on LBF (Seen, No Diversity) and GRF. Results averaged over random seeds (LBF: $5$; GRF: $20$ for \mixer{} and NOM, $10$ for AUX-PPO, MIX-SingleEnc).
    }
    \label{fig:auxppo}
\end{figure}

To address the first, we compare four methods that each add one ingredient: NOM (feed-forward PPO), AUX-PPO (feed-forward PPO with \mixer{}'s four prediction losses as auxiliary heads), MIX-SingleEnc (PPO with a single recurrent encoder, parameter-matched to \mixer{}'s four, trained on all four prediction objectives simultaneously), and \mixer{}. MIX-SingleEnc differs from \mixer{} only in combining the objectives in one shared encoder rather than in separate encoders weighted by the MoE module. \Cref{fig:auxppo} compares all four methods on LBF and GRF, two fully observable environments; \mixer{} outperforms the rest in both, so its advantage is not explained by auxiliary supervision, recurrence, or capacity alone. We believe one factor in \mixer{}'s outperformance to be the absence of gradient interference: when objectives share parameters, as in AUX-PPO and MIX-SingleEnc, their gradients compete within one network, whereas \mixer{} lets each shape its own representation, which the MoE module then combines for the downstream task.

\begin{figure}
  \centering
    \includegraphics[width=0.60\textwidth]
        {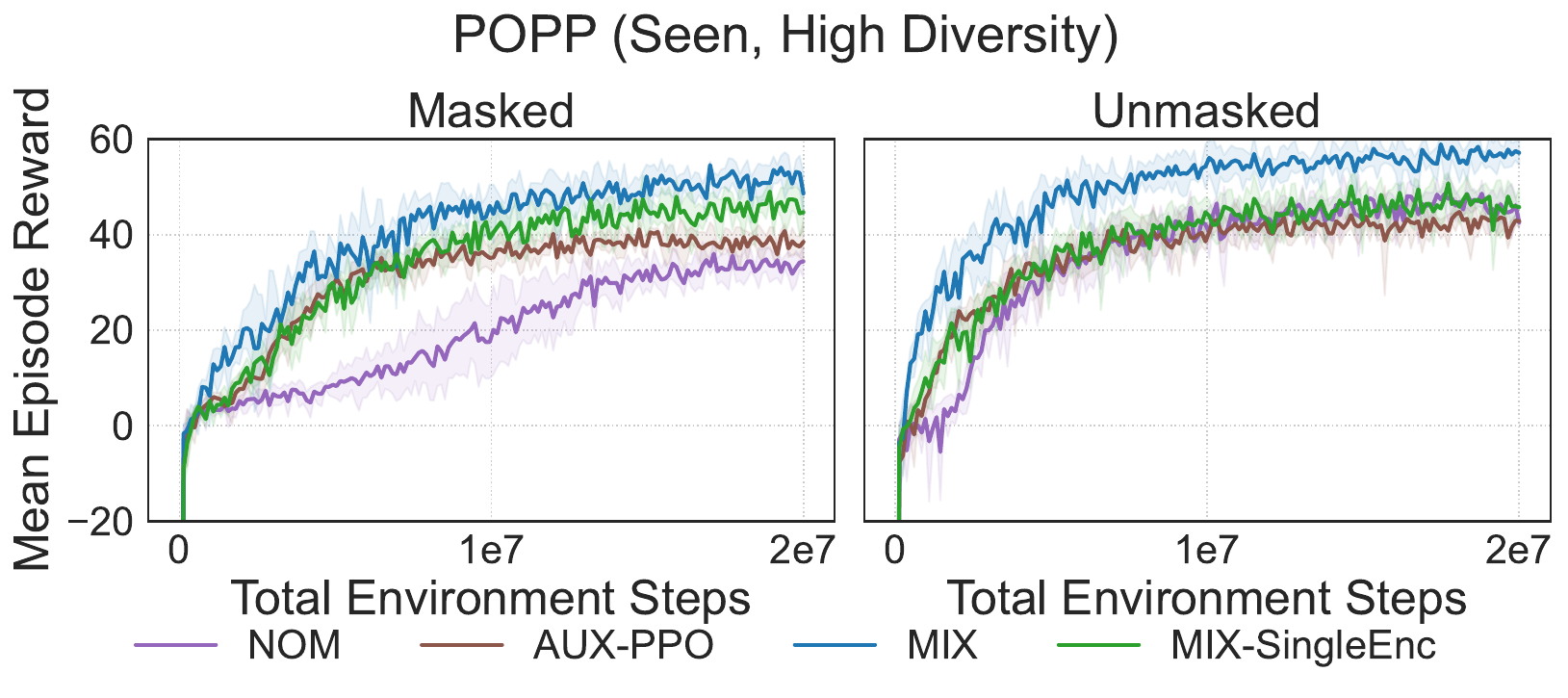}
    \caption{
    Comparing \mixer{} against MIX-SingleEnc, AUX-PPO and NOM on two variants of POPP (Seen, High Diversity): the masked variant, which is the environment reported throughout this paper, and the unmasked variant, which allows the ego-agent to observe the positions and velocities of its opponents regardless of distance, making the environment effectively fully observable. Results averaged over $5$ random seeds.
    }
    \label{fig:unmasked}
\end{figure}

The Partially Observable Predator-Prey environment can also be made fully observable by unmasking the positions and velocities of the predators in the prey's observation. \Cref{fig:unmasked} compares the four methods before and after unmasking. MIX-SingleEnc and AUX-PPO improve over NOM in the masked setting but no longer do once the opponent state is observable, whereas \mixer{} retains a clear advantage over all three methods in both settings. Together with \Cref{fig:auxppo}, this shows that the ability of AUX-PPO and MIX-SingleEnc's simpler architectures to exploit the auxiliary signals is inconsistent across settings, while \mixer{} realizes their benefit in every environment evaluated.

Turning to the second question, we argue that opponent modeling should remain beneficial even under full observability, as the latent opponent identity and policy are not recoverable from the state alone. In all three fully observable settings (LBF, GRF, and unmasked POPP), \mixer{} outperforms NOM, despite NOM having access to the same opponent state (notably, MIX-SingleEnc does not improve over NOM in unmasked POPP, so this gap is not explained by recurrence alone). Masking the opponent state in POPP widens this gap, consistent with \mixer{}’s recurrent encoders additionally reconstructing hidden opponent information to which the feed-forward NOM has no access. This interpretation is further supported by the probe analysis in \Cref{sec:embedding-eval}, showing that \mixer{}’s embedding carries opponent-discriminative information beyond that recovered by a randomly initialized encoder.

\section{Computational Complexity Analysis}
\label{sec:complexity-analysis}

To further contextualize the performance of \mixer{}, we provide a detailed computational complexity and runtime analysis. Metrics were evaluated on the POPP task using an Intel Xeon Platinum 8452Y CPU, with each run restricted to 4 CPU cores.

\begin{table}[h]
    \centering
    \caption{Computational complexity and runtime analysis. Forward FLOPs and inference time measure the cost of a single environment step, while total training time measures the wall-clock time required to complete training.}
    \label{tab:complexity}
    \begin{tabular}{lrrr}
        \toprule
        Model & Forward FLOPs & Inference time & Total training time \\
        \midrule
        NOM (No Opponent Modeling) & 49.3 M & 1.83 ms & 3.18 h \\
        LIAM & 53.4 M & 2.73 ms & 3.43 h \\
        OMG & 57.9 M & 2.94 ms & 8.49 h$^*$ \\
        MeLIBA & 51.8 M & 3.36 ms & 25.58 h$^\dagger$ \\
        \midrule
        \textbf{\mixer{} (Ours)} & \textbf{61.0 M} & \textbf{5.18 ms} & \textbf{4.32 h} \\
        \bottomrule
    \end{tabular}
    \vspace{4pt}\\
    \small{$^*$Includes the required pretraining phase. $^\dagger$Reflects the quadratic-complexity update pass discussed above.}
\end{table}

As shown in \Cref{tab:complexity}, \mixer{} introduces a modest increase in forward computation to enable its significant gains in representation capacity. Whereas the opponent modeling baselines each employ a single encoder, \mixer{} employs four, one per episode component. Although this naturally yields a higher inference latency ($5.18$ ms), this overhead is negligible in absolute terms: an inference time of $5.18$ ms supports an execution rate of nearly $200$ FPS, falling well within standard real-time MARL latency budgets. The total footprint of \mixer{} is $61.0$ M FLOPs, only slightly above that of OMG, as the bulk of the computational weight remains within the shared PPO backbone.

Furthermore, recent state-of-the-art baselines introduce severe training-time bottlenecks. Specifically, OMG requires a separate, costly pretraining phase, while MeLIBA relies on a quadratic-complexity update pass. By avoiding these structural bottlenecks, \mixer{} fully trains in just $4.32$ hours. This represents a significant improvement in training efficiency, saving over $4$ hours compared to OMG and $21$ hours compared to MeLIBA.

\begin{figure}
  \centering
    \includegraphics[width=\textwidth]
        {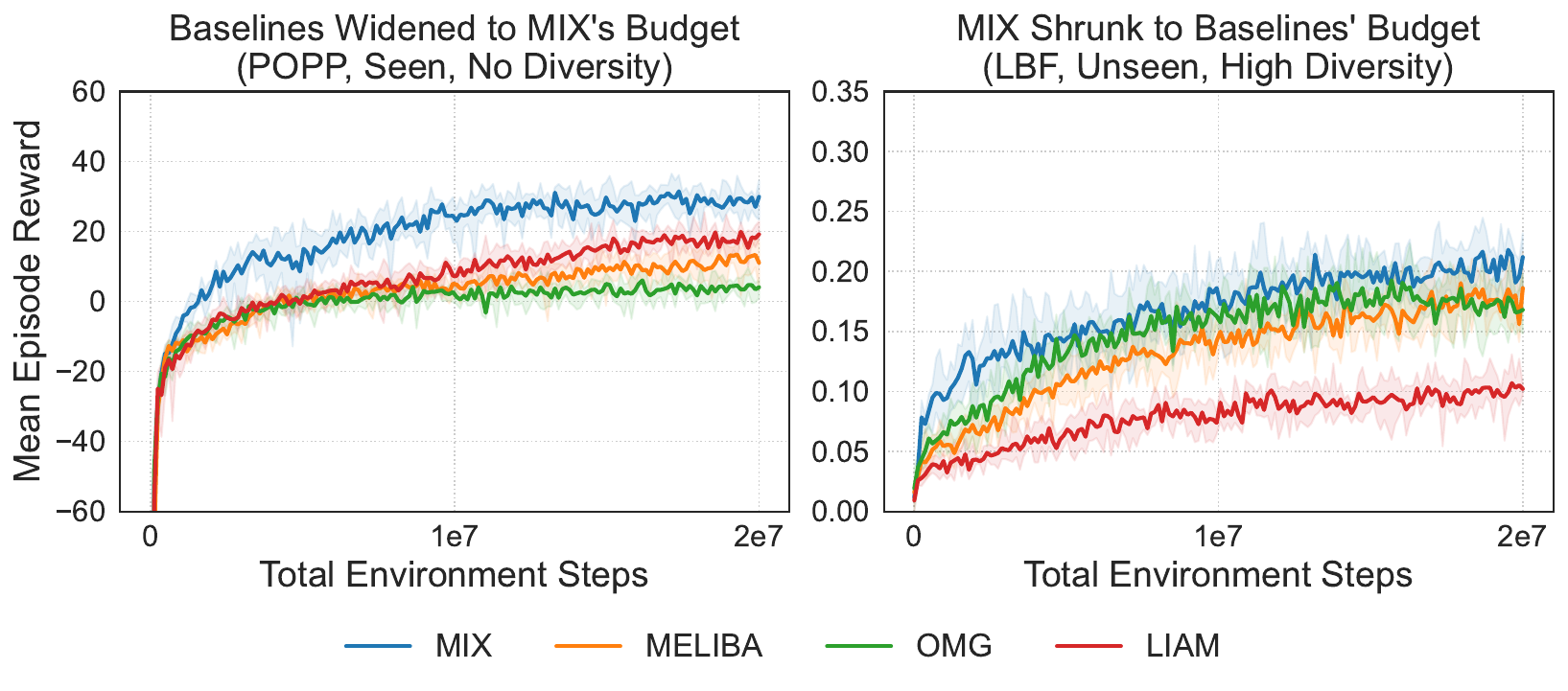}
    \caption{Capacity-matched comparisons. (Left) Baselines widened so that their opponent modeling parameter count matches \mixer{}'s, on POPP (seen, no diversity). (Right) \mixer{} with each encoder shrunk so that its total opponent modeling parameter count matches that of the baselines, on LBF (unseen, high diversity). NOM is omitted, as it has no opponent modeling component to resize. Results averaged over 5 random seeds.}
    \label{fig:param_matching_ablation}
\end{figure}

Each of \mixer{}'s four encoders is similar in size to each baseline's single encoder, so \mixer{} carries a higher parameter count for the opponent modeling component. We report all methods at their published architectures. We contend, however, that \mixer{}'s additional capacity does not explain its advantage, and run two ablations to establish this. First, we widen the baselines' encoders so that their total parameter count matches \mixer{}'s, lightly tuning each baseline's encoder learning rate at the new width. Second, we shrink each of \mixer{}'s encoders so that its parameter count roughly matches that of the baselines. \Cref{fig:param_matching_ablation} presents these two ablations on POPP (seen, no diversity) and LBF (unseen, high diversity) respectively. In both settings \mixer{} remains the best-performing method.

\section{Implementation Details}
\label{sec:implementation-details}

We use the TorchRL~\citep{bou2023torchrl} and VMAS~\citep{bettini2022vmas} frameworks to implement our experiments.
We implement all algorithms with PPO~\citep{schulman2017proximalpolicyoptimizationalgorithms} as the backbone, initialized identically. For LIAM~\citep{papoudakis_agent_2021}, we adopt the open-source implementation\footnote{\url{https://github.com/uoe-agents/LIAM}, MIT License}, and for OMG~\citep{yu2024}, we follow the supplementary material provided by the authors\footnote{\url{https://openreview.net/forum?id=Lt6wO0oZ8k}, CC BY 4.0 License}. As in their original work, we pretrain the VAE of OMG. In the partially observable environments Kuhn Poker and POPP, we ensure that the inputs to MeLIBA's hierarchical VAE and OMG's CVAE use only the ego-agent's own local observation (since the pretrained VAE in OMG is only used during training, it models global information $s_{t} \approx [o^1_{t}, o^{-1}_{t}]$ even under partial observability). The intention embedding size and encoder hidden layer size are consistent across all opponent modeling methods, and we train all algorithms for the same number of steps. We note that in GRF, $2$ of $20$ MeLIBA seeds diverged during the final $10\%$ of training. We exclude evaluation points after divergence, so the final portion of MeLIBA's GRF curve averages over $18$ seeds rather than $20$.

\begin{table}[ht]
\centering
\small
\caption{Hyperparameters used for the POPP, LBF, GRF and Kuhn Poker environments. The final column lists the values searched for the \mixer{}-specific hyperparameters, which were tuned per environment; see the text for the tuning procedure of all other hyperparameters.}
\label{tab:hyperparameters}
\begin{tabular}{l c c c c c}
\toprule
\textbf{Hyperparameter} & \textbf{POPP} & \textbf{LBF} & \textbf{GRF} & \textbf{Kuhn Poker} & \textbf{Search Range} \\
\midrule
\multicolumn{6}{l}{\textit{Environment \& Training Loop}} \\
\midrule
Max Episode Steps & 100 & 50 & 200 & 100 & \\
\# Opponents & 3 & 1 & 6 & 1 & \\
Total Iterations & 3340 & 3340 & 2500 & 400 & \\
Steps per Iteration & 6000 & 6000 & 8000 & 5000 & \\
Eval Frequency (iters) & 20 & 20 & 50 & 10 & \\
Eval Episodes & 200 & 200 & 200 & 10,000 & \\
\midrule
\multicolumn{6}{l}{\textit{PPO Algorithm}} \\
\midrule
Actor \& Critic Layers & 3 & 3 & 3 & 3 & \\
Layer Hidden Dim & 256 & 256 & 256 & 256 & \\
Learning Rate & $3 \times 10^{-4}$ & $3 \times 10^{-4}$ & $3 \times 10^{-4}$ & $1 \times 10^{-4}$ & \\
Entropy Coeff & 0.05 & 0.001 & 0.003 & 0.01 & \\
Gamma ($\gamma$) & 0.99 & 0.99 & 0.993 & 0.99 & \\
Lambda ($\lambda$) & 0.95 & 0.95 & 0.95 & 0.95 & \\
Clip Epsilon ($\epsilon$) & 0.2 & 0.2 & 0.2 & 0.2 & \\
Max Grad Norm & 0.5 & 0.5 & 0.64 & 0.5 & \\
Epochs & 5 & 5 & 5 & 5 & \\
Minibatch Size & 10 & 10 & 5 & 10 & \\
\midrule
\multicolumn{6}{l}{\textit{\mixer{}}} \\
\midrule
Learning Rate & $3 \times 10^{-4}$ & $3 \times 10^{-4}$ & $3 \times 10^{-4}$ & $5 \times 10^{-4}$ & \\
Embedding Dim & 20 & 20 & 32 & 8 & \\
Layer Hidden Dim & 128 & 128 & 128 & 128 & \\
Minibatch Size & 20 & 20 & 10 & 10 & \\
Epochs & 5 & 5 & 5 & 5 & \\
Future Rewards Horizon $H$ & 5 & 2 & 10 & 2 & $\{1, 2, 3, 5, 10\}$ \\
Future Rewards $\tau$ & 0.2 & 0.2 & 0.2 & 0.2 & $\{0.05, 0.1, 0.2\}$ \\
Future Rewards Decoder Dim & 5 & 5 & 5 & 5 & $\{2, 5, 10\}$ \\
Future States Horizon $n$ & 2 & 1 & 1 & 1 & $\{1, 2, 3\}$ \\
\bottomrule
\end{tabular}
\end{table}

PPO backbone hyperparameters were tuned per environment on the no-opponent-modeling baseline (NOM) and then held fixed across all methods including \mixer{}, so that no method benefits from a separately tuned policy-learning backbone (for GRF, we started from the PPO hyperparameters reported by \citet{kurach2020googleresearchfootballnovel} and lightly tuned a subset of them on NOM). We note that, given PPO’s hyperparameter sensitivity, per-method backbone tuning could shift the relative margins we report. Hyperparameters common to the opponent-modeling methods, such as the embedding dimension, were tuned on the LIAM baseline and held fixed across all methods. We then tuned \mixer{}-specific hyperparameters per environment: the future-rewards horizon, future-state horizon, InfoNCE temperature, and future-rewards decoder dimension. We have found \mixer{} robust to these parameters---\Cref{fig:hyperparameter_ablation} presents the ablation results for selected hyperparameters, with very little variation in performance; \Cref{tab:hyperparameters} presents the values used to generate our experiment results. Each baseline's method-specific hyperparameters were likewise swept per environment. Components pretrained prior to training, namely OMG’s subgoal prior VAE, use the settings from the original implementation and are shared across seeds for each opponent set.

To generate embeddings, \mixer{} encoders pass observation sequences through an LSTM layer~\citep{Hochreiter97} followed by a linear layer with a ReLU activation. A final linear layer projects these activations to the output dimension. The corresponding decoder maps these embeddings back to the target space using a three-layer MLP with ReLU activations. The encoder and decoder are trained jointly using the Adam optimizer.

\begin{figure}
  \centering
    \includegraphics[width=\textwidth]
        {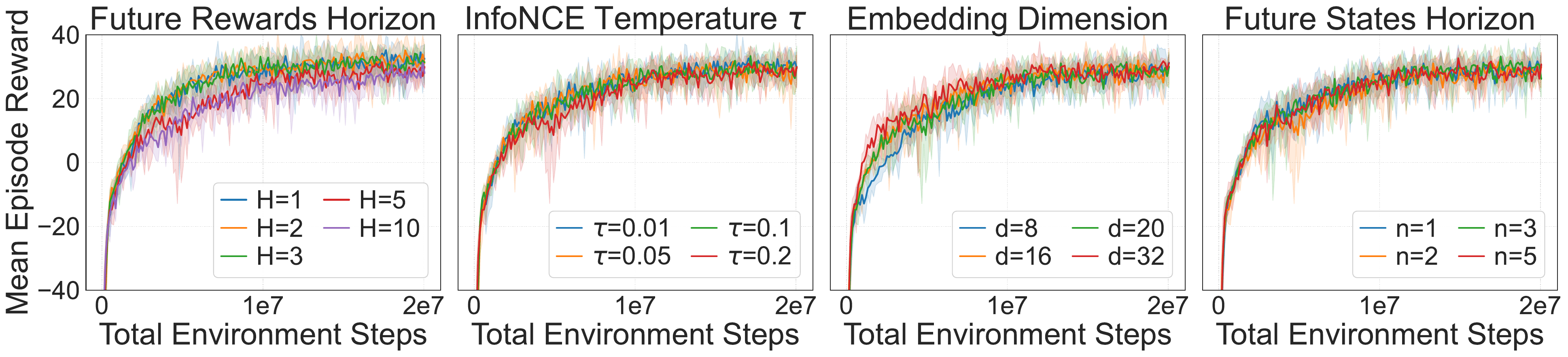}
    \caption{
    Ablations of MIX hyperparameters in the POPP environment (seen setting, no opponent diversity). Results averaged over $5$ random seeds.
    }
    \label{fig:hyperparameter_ablation}
\end{figure}

\section{Compute Resources}
\label{sec:compute-resources}
All experiments were run on Intel Xeon Platinum 8452Y processors, with each run restricted to 4 CPU cores. The development and evaluation of our method required approximately $40{,}000$ CPU-hours in total, broken down as follows:

\begin{itemize}[leftmargin=*,topsep=0ex,itemsep=0ex,parsep=0ex]
    \item Exploration \& Tuning: Preliminary experimentation and hyperparameter optimization required an estimated $10{,}000$ CPU-hours.
    \item Final Evaluation: Executing the final reported experiments across all environments and random seeds required approximately $30{,}000$ CPU-hours.
\end{itemize}

%% file: main.bib
@inproceedings{papoudakis_agent_2021,
	title = {Agent {Modelling} under {Partial} {Observability} for {Deep} {Reinforcement} {Learning}},
	volume = {34},
	urldate = {2026-01-14},
	booktitle = {Advances in {Neural} {Information} {Processing} {Systems}},
	author = {Papoudakis, Georgios and Christianos, Filippos and Albrecht, Stefano},
	year = {2021},
	pages = {19210--19222},
}

@inproceedings{yu2024,
author = {Yu, Xiaopeng and Jiang, Jiechuan and Lu, Zongqing},
title = {Opponent modeling based on subgoal inference},
year = {2024},
booktitle = {Proceedings of the 38th International Conference on Neural Information Processing Systems},
articleno = {1936},
numpages = {25},
}

@inproceedings{meliba,
author = {Zintgraf, Luisa and Devlin, Sam and Ciosek, Kamil and Whiteson, Shimon and Hofmann, Katja},
title = {Deep Interactive Bayesian Reinforcement Learning via Meta-Learning},
year = {2021},
booktitle = {Proceedings of the 20th International Conference on Autonomous Agents and MultiAgent Systems},
pages = {1712–1714},
numpages = {3},
}

@article{albrecht_autonomous_2018,
	title = {Autonomous agents modelling other agents: {A} comprehensive survey and open problems},
	volume = {258},
	shorttitle = {Autonomous agents modelling other agents},
	urldate = {2026-01-14},
	journal = {Artificial Intelligence},
	author = {Albrecht, Stefano V. and Stone, Peter},
	year = {2018},
	pages = {66--95},
}

@InProceedings{lili,
  title = 	 {Learning Latent Representations to Influence Multi-Agent Interaction},
  author =       {Xie, Annie and Losey, Dylan and Tolsma, Ryan and Finn, Chelsea and Sadigh, Dorsa},
  booktitle = 	 {Proceedings of the 2020 Conference on Robot Learning},
  pages = 	 {575--588},
  year = 	 {2021},
  volume = 	 {155},
}

@article{
shapley1953stochastic,
author = {L. S. Shapley },
title = {Stochastic Games},
journal = {Proceedings of the National Academy of Sciences},
volume = {39},
number = {10},
pages = {1095-1100},
year = {1953},
eprint = {https://www.pnas.org/doi/pdf/10.1073/pnas.39.10.1095}}

@inproceedings{fu_greedy_2022,
	title = {Greedy when {Sure} and {Conservative} when {Uncertain} about the {Opponents}},
	language = {en},
	urldate = {2026-01-14},
	booktitle = {Proceedings of the 39th {International} {Conference} on {Machine} {Learning}},
	author = {Fu, Haobo and Tian, Ye and Yu, Hongxiang and Liu, Weiming and Wu, Shuang and Xiong, Jiechao and Wen, Ying and Li, Kai and Xing, Junliang and Fu, Qiang and Yang, Wei},
	year = {2022},
	pages = {6829--6848},
}

@InProceedings{pmlr-v80-grover18a,
  title = 	 {Learning Policy Representations in Multiagent Systems},
  author =       {Grover, Aditya and Al-Shedivat, Maruan and Gupta, Jayesh and Burda, Yuri and Edwards, Harrison},
  booktitle = 	 {Proceedings of the 35th International Conference on Machine Learning},
  pages = 	 {1802--1811},
  year = 	 {2018},
  volume = 	 {80},
}

@inproceedings{omis,
author = {Jing, Yuheng and Liu, Bingyun and Li, Kai and Zang, Yifan and Fu, Haobo and Fu, Qiang and Xing, Junliang and Cheng, Jian},
title = {Opponent modeling with in-context search},
year = {2024},
booktitle = {Proceedings of the 38th International Conference on Neural Information Processing Systems},
articleno = {1967},
numpages = {43},
}

@inproceedings{
tacchetti2018relational,
title={Relational Forward Models for Multi-Agent Learning},
author={Andrea Tacchetti and H. Francis Song and Pedro A. M. Mediano and Vinicius Zambaldi and János Kramár and Neil C. Rabinowitz and Thore Graepel and Matthew Botvinick and Peter W. Battaglia},
booktitle={International Conference on Learning Representations},
year={2019},
}

@inproceedings{willi_cola_2022,
	title = {{COLA}: {Consistent} {Learning} with {Opponent}-{Learning} {Awareness}},
	shorttitle = {{COLA}},
	language = {en},
	urldate = {2026-01-19},
	booktitle = {Proceedings of the 39th {International} {Conference} on {Machine} {Learning}},
	author = {Willi, Timon and Letcher, Alistair Hp and Treutlein, Johannes and Foerster, Jakob},
	year = {2022},
	pages = {23804--23831},
}

@inproceedings{foerster_learning_2018,
author = {Foerster, Jakob and Chen, Richard Y. and Al-Shedivat, Maruan and Whiteson, Shimon and Abbeel, Pieter and Mordatch, Igor},
title = {Learning with Opponent-Learning Awareness},
year = {2018},
booktitle = {Proceedings of the 17th International Conference on Autonomous Agents and MultiAgent Systems},
pages = {122–130},
numpages = {9},
}

@InProceedings{pmlr-v162-lu22d,
  title = 	 {Model-Free Opponent Shaping},
  author =       {Lu, Christopher and Willi, Timon and De Witt, Christian A Schroeder and Foerster, Jakob},
  booktitle = 	 {Proceedings of the 39th International Conference on Machine Learning},
  pages = 	 {14398--14411},
  year = 	 {2022},
  volume = 	 {162},
}

@InProceedings{pmlr-v48-he16,
  title = 	 {Opponent Modeling in Deep Reinforcement Learning},
  author =       {He, He and Boyd-Graber, Jordan and Kwok, Kevin and Daum\'e, III, Hal},
  booktitle = 	 {Proceedings of The 33rd International Conference on Machine Learning},
  pages = 	 {1804--1813},
  year = 	 {2016},
  volume = 	 {48},
}

@inproceedings{Kingma2014,
  title     = {Auto-Encoding Variational Bayes},
  author    = {Kingma, Diederik P. and Welling, Max},
  booktitle = {Proceedings of the 2nd International Conference on Learning Representations (ICLR)},
  year      = {2014},
}

@misc{mnih2013playingatarideepreinforcement,
      title={Playing Atari with Deep Reinforcement Learning}, 
      author={Volodymyr Mnih and Koray Kavukcuoglu and David Silver and Alex Graves and Ioannis Antonoglou and Daan Wierstra and Martin Riedmiller},
      year={2013},
      eprint={1312.5602},
      archivePrefix={arXiv},
      primaryClass={cs.LG},
}

@article{premack_does_1978,
	title = {Does the chimpanzee have a theory of mind?},
	volume = {1},
	copyright = {https://www.cambridge.org/core/terms},
	language = {en},
	number = {4},
	urldate = {2025-05-29},
	journal = {Behavioral and Brain Sciences},
	author = {Premack, David and Woodruff, Guy},
	year = {1978},
	pages = {515--526},
}

@inproceedings{wu_too_2021,
author = {Wang, Rose E. and Wu, Sarah A. and Evans, James A. and Tenenbaum, Joshua B. and Parkes, David C. and Kleiman-Weiner, Max},
title = {Too Many Cooks: Coordinating Multi-agent Collaboration Through Inverse Planning},
year = {2020},
booktitle = {Proceedings of the 19th International Conference on Autonomous Agents and MultiAgent Systems},
pages = {2032–2034},
numpages = {3},
}

@inproceedings{huang_efficient_2024,
	title = {Efficient adaptation in mixed-motive environments via hierarchical opponent modeling and planning},
	booktitle = {Proceedings of the 41st {International} {Conference} on {Machine} {Learning}},
	author = {Huang, Yizhe and Liu, Anji and Kong, Fanqi and Yang, Yaodong and Zhu, Song-Chun and Feng, Xue},
	year = {2024},
}

@inproceedings{raileanu_modeling_2018,
	title = {Modeling {Others} using {Oneself} in {Multi}-{Agent} {Reinforcement} {Learning}},
	volume = {80},
	booktitle = {Proceedings of the 35th {International} {Conference} on {Machine} {Learning}},
	author = {Raileanu, Roberta and Denton, Emily and Szlam, Arthur and Fergus, Rob},
	year = {2018},
	pages = {4257--4266},
}

@inproceedings{rabinowitz_machine_2018,
	title = {Machine {Theory} of {Mind}},
	volume = {80},
	booktitle = {Proceedings of the 35th {International} {Conference} on {Machine} {Learning}},
	author = {Rabinowitz, Neil and Perbet, Frank and Song, Francis and Zhang, Chiyuan and Eslami, S. M. Ali and Botvinick, Matthew},
	year = {2018},
	pages = {4218--4227},
}

@inproceedings{zhao2017learning,
author = {Zhao, Shengjia and Song, Jiaming and Ermon, Stefano},
title = {Learning hierarchical features from deep generative models},
year = {2017},
booktitle = {Proceedings of the 34th International Conference on Machine Learning - Volume 70},
pages = {4091–4099},
numpages = {9},
}

@inproceedings{lowe2017,
author = {Lowe, Ryan and Wu, Yi and Tamar, Aviv and Harb, Jean and Abbeel, Pieter and Mordatch, Igor},
title = {Multi-agent actor-critic for mixed cooperative-competitive environments},
year = {2017},
booktitle = {Proceedings of the 31st International Conference on Neural Information Processing Systems},
pages = {6382–6393},
numpages = {12},
}

@inproceedings{bohmer2020,
author = {B\"{o}hmer, Wendelin and Kurin, Vitaly and Whiteson, Shimon},
title = {Deep coordination graphs},
year = {2020},
booktitle = {Proceedings of the 37th International Conference on Machine Learning},
articleno = {92},
numpages = {12},
}

@inproceedings{zhao2023,
author = {Zhao, Rui and Song, Jinming and Yuan, Yufeng and Hu, Haifeng and Gao, Yang and Wu, Yi and Sun, Zhongqian and Yang, Wei},
title = {Maximum entropy population-based training for zero-shot human-AI coordination},
year = {2023},
booktitle = {Proceedings of the Thirty-Seventh AAAI Conference on Artificial Intelligence and Thirty-Fifth Conference on Innovative Applications of Artificial Intelligence and Thirteenth Symposium on Educational Advances in Artificial Intelligence},
articleno = {689},
numpages = {9},
}

@misc{schulman2017proximalpolicyoptimizationalgorithms,
      title={Proximal Policy Optimization Algorithms}, 
      author={John Schulman and Filip Wolski and Prafulla Dhariwal and Alec Radford and Oleg Klimov},
      year={2017},
      eprint={1707.06347},
      archivePrefix={arXiv},
      primaryClass={cs.LG},
}

@inproceedings{christianos2020shared,
  title={Shared Experience Actor-Critic for Multi-Agent Reinforcement Learning},
  author={Christianos, Filippos and Schäfer, Lukas and Albrecht, Stefano V},
  booktitle = {Advances in Neural Information Processing Systems (NeurIPS)},
  year={2020}
}

@inproceedings{papoudakis2021benchmarking,
   title={Benchmarking Multi-Agent Deep Reinforcement Learning Algorithms in Cooperative Tasks},
   author={Georgios Papoudakis and Filippos Christianos and Lukas Schäfer and Stefano V. Albrecht},
   booktitle = {Proceedings of the Neural Information Processing Systems Track on Datasets and Benchmarks (NeurIPS)},
   year={2021},
   openreview = {https://openreview.net/forum?id=cIrPX-Sn5n},
}

@article{LanctotEtAl2019OpenSpiel,
  title     = {{OpenSpiel}: A Framework for Reinforcement Learning in Games},
  author    = {Marc Lanctot and Edward Lockhart and Jean-Baptiste Lespiau and
               Vinicius Zambaldi and Satyaki Upadhyay and Julien P\'{e}rolat and
               Sriram Srinivasan and Finbarr Timbers and Karl Tuyls and
               Shayegan Omidshafiei and Daniel Hennes and Dustin Morrill and
               Paul Muller and Timo Ewalds and Ryan Faulkner and J\'{a}nos Kram\'{a}r
               and Bart De Vylder and Brennan Saeta and James Bradbury and David Ding
               and Sebastian Borgeaud and Matthew Lai and Julian Schrittwieser and
               Thomas Anthony and Edward Hughes and Ivo Danihelka and Jonah Ryan-Davis},
  year      = {2019},
  eprint    = {1908.09453},
  archivePrefix = {arXiv},
  primaryClass = {cs.LG},
  journal   = {CoRR},
  volume    = {abs/1908.09453},
}

@misc{oord2019representationlearningcontrastivepredictive,
      title={Representation Learning with Contrastive Predictive Coding}, 
      author={Aaron van den Oord and Yazhe Li and Oriol Vinyals},
      year={2019},
      eprint={1807.03748},
      archivePrefix={arXiv},
      primaryClass={cs.LG},
}

@inproceedings{liu2023,
author = {Liu, Shunyu and Zhou, Yihe and Song, Jie and Zheng, Tongya and Chen, Kaixuan and Zhu, Tongtian and Feng, Zunlei and Song, Mingli},
title = {Contrastive identity-aware learning for multi-agent value decomposition},
year = {2023},
booktitle = {Proceedings of the Thirty-Seventh AAAI Conference on Artificial Intelligence and Thirty-Fifth Conference on Innovative Applications of Artificial Intelligence and Thirteenth Symposium on Educational Advances in Artificial Intelligence},
articleno = {1301},
numpages = {9},
}

@inproceedings{
lo2024learning,
title={Learning Multi-Agent Communication with Contrastive Learning},
author={Yat Long Lo and Biswa Sengupta and Jakob Nicolaus Foerster and Michael Noukhovitch},
booktitle={The Twelfth International Conference on Learning Representations},
year={2024},
}

@misc{jiang2024mixtralexperts,
      title={Mixtral of Experts}, 
      author={Albert Q. Jiang and Alexandre Sablayrolles and Antoine Roux and Arthur Mensch and Blanche Savary and Chris Bamford and Devendra Singh Chaplot and Diego de las Casas and Emma Bou Hanna and Florian Bressand and Gianna Lengyel and Guillaume Bour and Guillaume Lample and Lélio Renard Lavaud and Lucile Saulnier and Marie-Anne Lachaux and Pierre Stock and Sandeep Subramanian and Sophia Yang and Szymon Antoniak and Teven Le Scao and Théophile Gervet and Thibaut Lavril and Thomas Wang and Timothée Lacroix and William El Sayed},
      year={2024},
      eprint={2401.04088},
      archivePrefix={arXiv},
      primaryClass={cs.LG},
}

@inproceedings{
shazeer2017,
title={ Outrageously Large Neural Networks: The Sparsely-Gated Mixture-of-Experts Layer},
author={Noam Shazeer and Azalia Mirhoseini and Krzysztof Maziarz and Andy Davis and Quoc Le and Geoffrey Hinton and Jeff Dean},
booktitle={International Conference on Learning Representations},
year={2017},
}

@incollection{Kuhn1950,
  title     = {A Simplified Two-Person Poker},
  author    = {Kuhn, H. W.},
  booktitle = {Contributions to the Theory of Games, Volume I},
  series    = {Annals of Mathematics Studies},
  publisher = {Princeton University Press},
  year      = {1950},
  pages     = {97--103},
}

@article{jacobs1991,
    author = {Jacobs, Robert A. and Jordan, Michael I. and Nowlan, Steven J. and Hinton, Geoffrey E.},
    title = {Adaptive Mixtures of Local Experts},
    journal = {Neural Computation},
    volume = {3},
    number = {1},
    pages = {79-87},
    year = {1991},
    eprint = {https://direct.mit.edu/neco/article-pdf/3/1/79/812104/neco.1991.3.1.79.pdf},
}

@inproceedings{jing2025,
author = {Jing, Yuheng and Li, Kai and Liu, Bingyun and Fu, Haobo and Fu, Qiang and Xing, Junliang and Cheng, Jian},
title = {An open-ended learning framework for opponent modeling},
year = {2025},
booktitle = {Proceedings of the Thirty-Ninth AAAI Conference on Artificial Intelligence and Thirty-Seventh Conference on Innovative Applications of Artificial Intelligence and Fifteenth Symposium on Educational Advances in Artificial Intelligence},
articleno = {2590},
numpages = {9},
}

@article{southey2009,
author = {Southey, Finnegan and Hoehn, Bret and Holte, Robert C.},
title = {Effective short-term opponent exploitation in simplified poker},
year = {2009},
issue_date = {February  2009},
volume = {74},
number = {2},
journal = {Mach. Learn.},
pages = {159–189},
numpages = {31}
}

@inproceedings{yu2022,
author = {Yu, Chao and Velu, Akash and Vinitsky, Eugene and Gao, Jiaxuan and Wang, Yu and Bayen, Alexandre and Wu, Yi},
title = {The surprising effectiveness of PPO in cooperative multi-agent games},
year = {2022},
booktitle = {Proceedings of the 36th International Conference on Neural Information Processing Systems},
articleno = {1787},
numpages = {14},
}

@book{Oliehoek2016,
author = {Oliehoek, Frans A. and Amato, Christopher},
title = {A Concise Introduction to Decentralized POMDPs},
year = {2016},
publisher = {Springer Publishing Company, Incorporated},
edition = {1st},
}

@misc{bou2023torchrl,
      title={{TorchRL}: A data-driven decision-making library for PyTorch}, 
      author={Albert Bou and Matteo Bettini and Sebastian Dittert and Vikash Kumar and Shagun Sodhani and Xiaomeng Yang and Gianni De Fabritiis and Vincent Moens},
      year={2023},
      eprint={2306.00577},
      archivePrefix={arXiv},
      primaryClass={cs.LG}
}

@article{bettini2022vmas,
  title = {{VMAS}: A Vectorized Multi-Agent Simulator for Collective Robot Learning},
  author = {Bettini, Matteo and Kortvelesy, Ryan and Blumenkamp, Jan and Prorok, Amanda},
  year = {2022},
  journal={The 16th International Symposium on Distributed Autonomous Robotic Systems},
  publisher={Springer}
}

@article{Hochreiter97,
author = {Hochreiter, Sepp and Schmidhuber, J\"{u}rgen},
title = {Long Short-Term Memory},
year = {1997},
issue_date = {November 15, 1997},
publisher = {MIT Press},
volume = {9},
number = {8},
journal = {Neural Comput.},
pages = {1735–1780},
numpages = {46}
}

@misc{shrikumar2017justblackboxlearning,
      title={Not Just a Black Box: Learning Important Features Through Propagating Activation Differences}, 
      author={Avanti Shrikumar and Peyton Greenside and Anna Shcherbina and Anshul Kundaje},
      year={2017},
      eprint={1605.01713},
      archivePrefix={arXiv},
      primaryClass={cs.LG},
}

@misc{kurach2020googleresearchfootballnovel,
      title={Google Research Football: A Novel Reinforcement Learning Environment}, 
      author={Karol Kurach and Anton Raichuk and Piotr Stańczyk and Michał Zając and Olivier Bachem and Lasse Espeholt and Carlos Riquelme and Damien Vincent and Marcin Michalski and Olivier Bousquet and Sylvain Gelly},
      year={2020},
      eprint={1907.11180},
      archivePrefix={arXiv},
      primaryClass={cs.LG},
}
